\documentclass[11pt]{article}
\usepackage[margin=1in]{geometry}
\pdfoutput=1
\usepackage{ifthen}
\newboolean{usenatbib}
\setboolean{usenatbib}{false}
\ifthenelse{\boolean{usenatbib}}{
  \usepackage[numbers,sort]{natbib}
}{
  \usepackage[style=alphabetic,maxalphanames=4,minalphanames=3,maxcitenames=2,maxbibnames=99,giveninits=false,doi=false,url=false,natbib]{biblatex}
  \addbibresource{refs.bib}

  \DeclareDelimFormat[bib]{nametitledelim}{\addspace}
}
\usepackage{bm}
\usepackage{fullpage}
\usepackage[english]{babel}
\usepackage[utf8]{inputenc}
\usepackage[T1]{fontenc}
\usepackage{ae} \usepackage{aecompl}
\usepackage{enumitem}
\setlist{nosep}
\usepackage{microtype}
\usepackage{booktabs}
\usepackage{hhline}
\usepackage{makecell}
\usepackage{amsfonts}
\usepackage{etoolbox}
\usepackage{enumitem}
\apptocmd{\sloppy}{\hbadness 10000\relax}{}{}

\usepackage{amsmath,amsthm,amssymb,mathtools,thmtools}
\usepackage{graphicx}
\usepackage{color}
\usepackage[dvipsnames]{xcolor}
\usepackage[textsize=scriptsize,disable]{todonotes}
\usepackage[colorlinks=true, allcolors=magenta]{hyperref}
\usepackage{xfrac}
\usepackage[normalem]{ulem}

\usepackage{algorithm}
\usepackage{setspace}
\usepackage[noend]{algpseudocode}

\hypersetup{colorlinks,
            breaklinks=true,
            linkcolor=purple,
            citecolor=purple,
            urlcolor=magenta,
            linktocpage,
            plainpages=false,
            }

\usepackage{macros}

\usepackage{comment}
\usepackage{tikz,pgfplots,pgfplotstable}
\usetikzlibrary{arrows.meta,positioning}
\usepgfplotslibrary{units}
\usepackage{tabularx} 
\usepackage{subfig}
\usepackage{array}
\usepgfplotslibrary{groupplots}
\usepackage{times}

\usepackage{ifthen}
\newboolean{arxivsubmission}
\setboolean{arxivsubmission}{true}

\title{Learning with Multiple Correct Answers - Regret Bounds under Different Feedback Models}

\author{
Alireza F. Pour \thanks{Equal contribution.}
\quad
Farnam Mansouri \footnotemark[1]
\quad
Shai Ben-David
\\[0.5em]
\small
University of Waterloo \& Vector Institute  \\
\footnotesize
\texttt{\{alireza.fathollahpour,f5mansou,shai\}@uwaterloo.ca}
}

\numberwithin{equation}{section}

\begin{document}
\maketitle

\begin{abstract}
We study the problem of learning with multiple correct answers, where each instance admits a set of valid labels. We primarily focus on the online setup, where in each round the learner must output a valid label for the queried example. This setting is motivated by language generation, in which a prompt may admit many acceptable completions, but not every completion is acceptable.
We study this problem under three feedback models. For each model, we characterize the optimal mistake bound in the realizable setting using an appropriate combinatorial dimension. We then show that the rate of regret can be constant, linear, or sublinear across the three models in the agnostic setting. Our results also imply sample complexity bounds for the batch setup that depend on the respective combinatorial dimensions.
\end{abstract}

\section{Introduction}

Many modern prediction tasks admit multiple correct outputs for a single input. A prominent example is language generation, where a prompt may have many acceptable completions but not all completions are acceptable. This phenomenon arises in settings such as open-ended question answering, dialogue generation, and instruction following, where correctness is better viewed as membership in a set of valid responses rather than by a single target label \cite{langford2007epoch, li-etal-2017-adversarial, ouyang2022training}. Similar considerations arise in recommendation systems and information retrieval, where multiple outputs may equally satisfy the task requirements \cite{joachims2002optimizing, liu2009learning}. These settings call for learning formulations in which correctness is inherently set-valued rather than single-valued.

In this work we analyze and compare learning with multiple correct answers with respect to three aspects; Online vs Batch learning, Realizable vs Agnostic setups, and the type of feedback available to the learner. We focus mainly on the online setting. An online problem proceeds in rounds. In each round, the learner is presented with an example and its objective is to produce a correct answer for that example. This setup naturally accommodates the following feedback models for the learner:

\begin{itemize}
\item \emph{Mistake-unknown feedback}, where the learner observes only a single correct label in each round (as studied by \citet{joshi2025learning});
\item \emph{Mistake-known feedback}, where, in addition to a single correct label, the learner also observes whether its prediction was correct; and
\item \emph{Set-valued feedback}, where the full set of correct labels is revealed to the learner (as studied by \citet{raman2024online}).\footnote{
One may also consider a fourth model of feedback, where the learner only gets a single bit indicating whether a mistake has been made. Such feedback is very weak since the learner can be forced to err on a single instance multiple times depending on the size of the label space. We therefore chose not to discuss this model here.}

\end{itemize}

Distinguishing between feedback models is crucial, as different applications provide different forms of supervision. \azreplace{For example}{A notable example is} large language model (LLM) training \azreplace{data often consist of}{where data contains} prompts paired with a single demonstrated completion. During interaction with users, the model may additionally receive feedback indicating whether a generated completion is acceptable, corresponding to the mistake-known setting.
Set-valued feedback also helps understanding the limits of a learner under maximal information feedback.

This work studies learning from a class of multi-label hypotheses, that is, a set of functions $h:\cX \rightarrow 2^{\cY}$ that map every instance in the domain $\cX$ to a subset of the label space $\cY$. We provides complete combinatorial (tree-based) characterizations of the optimal mistake bounds for any multi-label class in the realizable case under all three feedback models. In the agnostic case, we demonstrate three different rates of regret among the feedback models, which is in contrast to single-valued online learning \citep{ben2009agnostic,hanneke2023multiclass}. Our results further imply sample complexity bounds for the batch setup for any hypothesis class. 
\paragraph{Our Contributions.}  We formalize learning with multiple correct answers in the online setting and under the three models of feedback (Section~\ref{sec:setup}).
\begin{itemize}[leftmargin=*]
    \item 

{\bf Realizable online learning.} We introduce novel combinatorial parameters $\LDU$, $\LDK$, and $\LDS$ and \fmreplace{show they characterize realizable online learnability under the three feedback models (Section~\ref{sec:LDs}}{based on these parameters we establish the first characterization of realizable online learnability under the three feedback models
(Section~\ref{sec:LDs})}.\fmdelete{To the best of our knowledge, is the first work to study the mistake-known and set-valued settings} \fmmargincomment{This is not the first work that studied these settings}\azedit{The mistake-unknown setting was studied by \citet{joshi2025learning}, but only for finite hypothesis classes. Our characterization extends their result from finite (only) to any hypothesis classes.} 

    \item
    {\bf Agnostic online learning.} We provide examples of three different regret rates across the three feedback models: linear in $T$, $O(\sqrt{T \log (T)})$, and constant.
    \begin{itemize}
        \item {\bf Mistake-unknown.} We show that in the agnostic setting, the regret of mistake-unknown multi-label learning can grow linearly in the number of rounds $T$, even for a hypothesis class of size $3$ with $\LDU = 1$ (Section~\ref{sec:agnostic-H3}). 
   \emph{ This highlights a distinction between realizable and agnostic learning: characterizations of the realizable setting fail to capture any meaningful regret bounds in the agnostic setting. This is in contrast to many learning settings where realizable learning implies agnostic learning \citep{hopkins2022realizable,ben2009agnostic}.}

   \item
   {\bf Mistake-known.}  We prove a novel upper bound of $O\left(\sqrt{\LDK(\cH)|\cY|T\log(T)}\right)$ on the regret of agnostic mistake-known multi-label learning of any class  (Theorem~\ref{thm:agnostic_mistake_unknown}). This is only $\tilde{O}(\sqrt{|\cY|\log(T)})$ away from known lower bounds for single-label setting which also apply here \citep{daniely2015multiclass}.

\item
{\bf Set-valued.} We show that the regret of agnostic set-valued multi-label learning for a variety of classes with $\LDS \fmreplace{>}{\geq} 1$ (including the hypothesis class of size $3$ mentioned above) admits a constant upper bound. 
    We also attain a 
    \(
    O\!\left(\sqrt{\LDS(\cH)T \log(T)}\right)
    \)
    upper bound for general hypothesis classes (Section~\ref{section:agnostic_set_valued}). \emph{This variety of regret rates stands in contrast to single-valued learning where \emph{any} hypothesis class must have a regret of order $\Omega(\sqrt{T})$ \citep{daniely2015multiclass}.}

    \end{itemize}

    \item
    {\bf Batch learning.} \emph{Our results yield novel sample-complexity upper bounds for the corresponding batch learning problems.
  We get the first known upper bound for the loss for batch learning in the full-set feedback model. We also improve the result of \citet{joshi2025learning} for the case of single correct answer from finite classes to any (possibly infinite) classes.}
  \ifthenelse{\boolean{arxivsubmission}}{The formal setup and our results for the batch problem are defined in Section~\ref{sec:batch}.}{Due to space constraints, we defer the formal setup and our results for the batch problem to Appendix~\ref{sec:batch}.}

\end{itemize}

\begin{remark}
    For the specific case of single-label class $\cH$, all the three feedback models will be equivalent to the multiclass online learning, where the mistake bound is characterized by the Littestone dimension, $\LD(\cH)$ \citep{daniely2015multiclass,hanneke2023multiclass}, and therefore, $ \LDU(\cH) = \LDK(\cH)  = \LDS(\cH) =\LD(\cH)$. In the multiclass setting, the regret is $\tilde\Theta(\sqrt{\LD(\cH)T})$ while we show the regret of a multi-label class can be independent of the corresponding Littlestone dimensions. Our work therefore suggests that the landscape of agnostic multi-label learning has a rich structure and may require combinatorial objects beyond those that characterize realizable learnability.
\end{remark}
   
\ifthenelse{\boolean{arxivsubmission}}{\ifthenelse{\boolean{arxivsubmission}}{\subsection{Related Works}}{\section{Related Works}}

Recently, \citet{joshi2025learning} and \citet{raman2024online} studied learning with multiple correct answers under the mistake-unknown and set-valued feedback models, respectively. Both works focus primarily on the online learning setting, with \citet{joshi2025learning} additionally considering implications for batch learning.

In contrast to our work, \citet{joshi2025learning} restrict attention to finite hypothesis classes, whereas we allow for possibly infinite classes. For the realizable setting, their framework is similar to our mistake-unknown feedback model. However, they only derive an upper bound for finite classes in this case, whereas we provide a full (upper and lower bounds) characterization. In the agnostic setting, however, the loss measures used in \citet{joshi2025learning} differ from ours, and as a result our guarantees are not directly comparable. We discuss these differences and the advantages of our setup in detail in Section~\ref{sec:regrets}.

Turning to \citet{raman2024online}, their notion of regret is again different from ours (see Section~\ref{sec:regrets}). 
Furthermore, we also show that in the agnostic set-valued setting the regret of a variety of hypothesis classes can be constant. There is no similar result in previous work.

At a higher level perspective, beyond multi-label classification, our results are situated
within the analysis of online learnability and its combinatorial characterizations,
originating from binary classification and the Littlestone dimension
\cite{littlestone1988learning,ben2009agnostic}, and extending to multi-class learning
\cite{daniely2015multiclass,hanneke2023multiclass}, transductive learning
\cite{ben1995online,hanneke2023trichotomy,hanneke2024multiclass}, and frameworks for online
learning with general (local) loss functions \cite{rakhlin2015online,blanchard2022universal,raman2025unified}. 
\citet{moran2023list} study an opposite prediction paradigm, in which there is a single correct
label per instance, but the learner is allowed to output a list of labels, where the loss is
defined by whether the correct label is contained in the list.

}{}

\section{Learning Setup}\label{sec:setup}
Let $\cX$ be a domain set, $\cY$ a finite set of labels and $T$ a natural number. In the online multi-label learning, an adversary and a learner play a sequential game for $T$ rounds. In each round $t \in [T]$, the adversary selects an instance $x_t \in \cX$ and a set of acceptable labels $S_t \subseteq \cY$
and reveals $x_t$ to the learner. The learner makes a (potentially randomized) prediction $\hat{y}_t \in \cY$ and suffers loss
$\id[\hat y_t \notin S_t].$
The learner's goal is to minimize its cumulative loss over the $T$ rounds.

We consider the following three feedback models:
for each round $t \in [T]$,

\begin{itemize}
    \item \textbf{(Mistake-unknown).}
    In addition to choosing $(x_t, S_t)$, the adversary also selects a predetermined label $y_t \in S_t$ (before observing $\hat{y}_t$) and reveals it to the learner. 

    \item \textbf{(Mistake-known).}
    In addition to revealing the predetermined label $y_t$, the adversary also reveals a binary indicator specifying whether the learner made a mistake, namely $\id[\hat y_t \notin S_t]$.

    \item \textbf{(Set-valued).}
    The adversary reveals the entire set $S_t$. 
\end{itemize}

In all cases, the feedback is revealed only after the learner has chosen $\hat y_t$.

A multi-label hypothesis class $\cH$ is a set of functions mapping $\cX$ to $2^{\cY}$. Denote by $\Pi(\cY)$ the set of probability measures on $\cY$.
 An online learner $\cA: (\cX \times \cZ)^* \rightarrow \Pi(\cY)^\cX$ is a mapping that each time $t$, gets as input the past sequence of instances $x_1,\ldots,x_{t-1}$ and feedbacks $z_1,\ldots,z_{t-1}$ along with a new test point $x_t$ and outputs a random variable $p_{\cA}(t):= \cA(x_1,z_1,\ldots,x_{t-1},z_{t-1})(x_t)$ over $\cY$. Depending on the feedback model, $\cZ$ is always chosen as $\cY$, $\cY \times \{0,1\}$, or $2^{\cY}$.

 We say that the multi-label learning in the online setting satisfies the realizability assumption if there exists some $h^* \in \cH$ such that for all $t \in [T]$ we have $h^*(x_t) = S_t$. Our primary focus in the realizable setting is the mistake bound (i.e., the cumulative loss) of the learner. Throughout our analysis of the realizable online learning, we restrict attention to deterministic learning algorithms.
 
 We now give the definitions of mistake bound for the three feedback models. In the mistake-unknown setting, we have $\cZ = \cY$ and, for any $t\in[T]$, the feedback $z_t$ is a (single) label $y_t \in S_t$.
\begin{definition} [Unknown Mistake Bound] 
Unknown mistake bound of a deterministic learning algorithm $\cA: (\cX \times \cY)^* \rightarrow \cY^{\cX}$ with respect to a 
hypothesis class $\cH$ and length $T \in \bN$ is defined as 
\begin{equation} \label{eq:mistake-unknown}
\rM^{\unknown}_{\cA}(T, \cH):=\sup_{\substack{ h^* \in \cH, \left(x_1, y_1\right), \ldots,\left(x_T, y_T\right):   \\  y_t \in h^*(x_t) }  } \expect{\sum_{t=1}^T \id\left[p_\cA(t)\notin h^*(x_t) \right]}.
\end{equation}
\end{definition}

 In the mistake-known setting, where $\cZ = \cY \times \{0,1\}$, the learner also observes as feedback whether it made a mistake. Formally, $\forall t\in[T]$, $z_t = (y_t,b_t)$, where $b_t:= \id[p_{\cA}(t) \notin S_t]$ is the mistake bit.

\begin{definition}[Known Mistake Bound] 

Known mistake bound of a deterministic learning algorithm $\cA: (\cX \times \cY \times \{0, 1\})^* \rightarrow \cY^{\cX}$ with respect to 
hypothesis class $\cH$ and length $T \in \bN$ is defined as 
\begin{equation} \label{eq:mistake-known}
    \rM^{\known}_{\cA}(T, \cH):=\sup_{\substack{ h^* \in \cH, \left(x_1, y_1\right), \ldots,\left(x_T, y_T\right):   \\  y_t \in h^*(x_t) }  }  \expect{\sum_{t=1}^T \id\left[p_{\cA}(t) \notin h^*(x_t) \right]}.
\end{equation}
\end{definition}

Finally, in the set-valued domain, the learner is given full information. Particularly, we have $\cZ = 2^{\cY}$ and at each time $t\in[T]$, $z_t=S_t$ and the entire label set is revealed.

\begin{definition} [Set-Valued Mistake Bound] 
Set-valued mistake bound of a deterministic learning algorithm $\cA: (\cX \times 2^\cY)^* \rightarrow \cY^{\cX}$ with respect to a 
hypothesis class $\cH$ and length $T \in \bN$ is defined as 
\begin{equation} \label{eq:mistake-set}
\rM^{\set}_{\cA}(T, \cH):=\sup_{h^*\in\cH,\,\left(x_1, h^*(x_1)\right), \ldots,\left(x_T, h^*(x_T)\right)}   \expect{\sum_{t=1}^T \id\left[p_\cA(t) \notin h^*(x_t) \right]}.
\end{equation}
\end{definition}

We will also consider the agnostic online multi-label learning.
In this setting, we will define the \emph{regret} of an online algorithm after $T$ rounds under the three different feedback models we consider.

\begin{definition} [Mistake-Unknown Regret] \label{def:regret-unknown}
 Mistake-unknown regret of a learning algorithm $\cA: (\cX \times \cY)^* \rightarrow \Pi(\cY)^\cX$ with respect to a multi-label hypothesis class $\cH$ and length $T \in \bN$ is defined as
\begin{equation} \label{eq:regret-unknown}
     \rR^{\unknown}_{\cA}(T, \cH):=\sup_{\substack{\left(x_1, y_1,S_1\right), \ldots,\left(x_T, y_T,S_T\right): \\ y_t \in S_t}} \expect{\sum_{t=1}^T \id [p_{\cA}(t)\notin S_t] - \inf_{h\in \cH} \sum_{t=1}^T \id\left[h(x_t) \neq S_t\right]}.
\end{equation}
\end{definition}

\begin{definition}[Mistake-Known Regret] \label{def:regret-known}
 Mistake-known regret of a learning algorithm $\cA: (\cX \times \cY \times \{0, 1\})^* \rightarrow \Pi(\cY)^\cX$ with respect to a multi-label hypothesis class $\cH$ and length $T \in \bN$ is defined as
\begin{equation} \label{eq:regret-known}
    \rR^{\known}_{\cA}(T, \cH):=\sup_{\substack{\left(x_1, y_1,S_1\right), \ldots,\left(x_T, y_T,S_T\right): \\ y_t \in S_t}} \expect{ \sum_{t=1}^T \id\left[p_{\cA}(t) \notin S_t \right] - \inf_{h\in \cH} \sum_{t=1}^T \id \left[ h(x_t) \neq S_t\right]}.
\end{equation}
\end{definition}

\begin{definition}[Set-Valued Regret] \label{def:regret-set}
 Set-valued regret of a learning algorithm $\cA: (\cX \times 2^{\cY})^* \rightarrow \Pi(\cY)^\cX$ with respect to a multi-label hypothesis class $\cH$ and length $T \in \bN$ is defined as
\begin{equation} \label{eq:regret-set}
    \rR^{\set}_{\cA}(T, \cH):=\sup_{\left(x_1,S_1\right), \ldots,\left(x_T, S_T\right) } \expect{ \sum_{t=1}^T \id\left[p_{\cA}(t) \notin S_t \right] - \inf_{h\in \cH} \sum_{t=1}^T \id\left[h(x_t) \neq S_t\right]}.
\end{equation} 
\end{definition}

\textbf{Notations.} For any $U = (u_i)_{i=1}^m$ and $t \in [m]$ we denote $U_{<t} := (u_i)_{i=1}^{t-1}$, and $U_{\leq t} := (u_i)_{i=1}^{t}$.

\subsection{A Comparison Between Notions of Regret} \label{sec:regrets}

\citet{joshi2025learning} introduce the following notion of regret for online mistake-unknown multi-label learning with respect to a multi-label hypothesis class $\cH$:
\begin{equation}\label{eq:regret_joshi}
    \sup_{(x_1,y_1), \ldots, (x_T,y_T)}
\expect{ \sum_{t=1}^T
\id\!\left[  p_{\cA}(t) \notin h^*(x_t)\right] - \sum_{t=1}^T\id\!\left[y_t \notin h^*(x_t)\right] },
\end{equation}
where $h^* \in \arg \min_{h \in \cH} \sum_{t=1}^T\id\!\left[y_t \notin h(x_t)\right]$. Observe, however, that the term $\id\!\left[p_{\cA}(t) \notin h^*(x_t)\right]$ measures the error of the algorithm \azreplace{\emph{relative to the best hypothesis} $h \in \cH$, rather than with respect to the underlying set of acceptable labels. In contrast, we are interested in measuring the learner’s error with respect to the true loss sequence, which motivates Definition~\ref{def:regret-unknown}}{relative to the hypothesis that contains the correct label more often—even though such a hypothesis may also include many incorrect labels. Thus, competing with such a hypothesis is not necessarily well aligned with the underlying prediction task.
Instead, we are interested in evaluating hypotheses with respect to the underlying set of acceptable labels, which motivates measuring the learner's error against the true loss sequence, i.e., $\id\!\left[p_{\cA}(t) \notin S_t\right]$, as in Definition~\ref{def:regret-unknown}. This definition also extends naturally to the mistake-known and set-valued settings while notions defined as in Equation~\eqref{eq:regret_joshi} do not admit such a direct generalization.} 

Furthermore, for a single-label hypothesis class $\cC \subseteq \cY^{\cX}$, 
\citet{raman2024online} introduced the following notion of regret for online set-valued multi-label learning relative to $\cC$:
\[
\sup_{(x_1,S_1), \ldots, (x_T,S_T)}
\expect{ \sum_{t=1}^T 
\id\!\left[  p_{\cA}(t) \notin S_t\right] - \inf_{c \in \cC}  \sum_{t=1}^T \id\!\left[ c(x_t) \notin S_t\right] }.
\]
In this formulation, while the feedback is set-valued, the hypothesis class itself remains single-labeled. When multiple labels are acceptable for a given instance, it is natural to consider hypothesis classes that may themselves assign multiple correct labels, as in our setup.

At first glance, an alternative formulation of regret to
Definition~\ref{def:regret-unknown} would be
\begin{equation}\label{eq:regret-alt}
\sup_{\substack{(x_1,y_1,S_1),\ldots,(x_T,y_T,S_T):\\ y_t\in S_t}}
\expect{
\sum_{t=1}^T \id\!\left[p_{\cA}(t)\notin S_t\right]
\;-\;
\inf_{h\in\cH}\sum_{t=1}^T \id\!\left[h(x_t)\nsubseteq S_t\right]
}.
\end{equation}
The principle of replacing ``not equal ($\neq$) '' with ``not a subset of ($\nsubseteq$)'' can be applied
to the mistake-known and set-valued feedback models too. The following example shows
that this notion of regret makes the adversary overly powerful and can render the
feedback $y_t$ completely non-informative.

\begin{example}
Consider the hypothesis class $\cH=\{h_1,h_2\}$ over the domain $\cX=\{x\}$ and
label set $\cY=[100]$, where $h_1(x)=\{1\}$ and $h_2(x)=\{2\}$. Let $h^*= h_b$ where $b$ is chosen uniformly at random from $\{1,2\}$. In each round $t$, we select $r$ uniformly at random from $\{3,\ldots,100\}$ and set $S_t=\{b,r\}, y_t=r$. We have
\[
\inf_{h\in\cH}\sum_{t=1}^T \id\!\left[h(x_t)\nsubseteq S_t\right]=0.
\]
Here, the feedback $y_t$ is completely
non-informative under both the mistake-known and mistake-unknown models. In particular, in the mistake-unknown model, the probability that $b=1$ is $1/2$ and the probability of selecting any particular $r\in[3,100]$ is $1/98$.
Consequently, the notion of regret defined in~\eqref{eq:regret-alt} is always at least $T/2$.
\end{example}

\ifthenelse{\boolean{arxivsubmission}}{\section{Combinatorial Parameters} \label{sec:LDs}}{\section{Combinatorial Parameters and Realizable Online Multi-Label Learning} \label{sec:LDs}}

We introduce combinatorial dimensions for each feedback model which are defined via trees that reflect the game between the learner and an adversary. We note that the dimension we introduce for the set-valued feedback model (Definition~\ref{def:set-multi-littlestone}) closely matches the set Littlestone dimension studied by \citet{raman2024online}.

 A set-valued tree $\cT$ is a complete $\cX \times 2^\cY$ valued, $\cY-$arry tree. Any path originating from root in this tree is defined by a sequence $\sigma = (\hat y_1, ..., \hat y_t) \in \cY^t$. Denote by $\nu(\sigma) = (\nu_{x}(\sigma),\nu_{S}(\sigma)) $ the final node in path $\sigma$.
 \begin{definition} [Set-Valued Littlestone Dimension] \label{def:set-multi-littlestone}
     Let $\cT$ be a complete set-valued tree with depth $d$. We say $\cT$ is shattered by $\cH$ if, for any path $\sigma = (\hat y_1, ..., \hat y_d)$ from root to some leaf, there exists a $h_\sigma \in \cH$ such that 
   for each $t \in [d]$, we have 
   \begin{enumerate}[label = (\roman *).]
    \item \emph{Consistency:} 
        $\nu_S(\sigma_{\le t}) = h_\sigma\bigl(\nu_x(\sigma_{< t})\bigr)$,
    \item \emph{Mistake requirement:} $\hat y_{t} \notin h_\sigma(\nu_{x}(\sigma_{< t}))$.
\end{enumerate}
   The Littlestone dimension $\LDS$ of $\cH$ in the set-valued online setting is then defined as the depth of the largest shattered tree. 
 \end{definition}

It is worth observing that in the above definition, $\nu_{x}(\sigma_{< t})$ is the unlabeled instance sitting at the parent of node $\nu(\sigma_{\leq t})$.

  A multi-label tree $\cT$ is a $\cX \times \cY$ valued, $\cY-$arry tree such that each node either has no children (a leaf) or $|\cY|$ children labeled for all $y \in \cY$. Any path originating from root in this tree is defined by a sequence $\sigma = (\hat y_1, ..., \hat y_t) \in \cY^t$. Denote $\nu(\sigma) = (\nu_{x}(\sigma),\nu_{y}(\sigma)) $ to be the final node in path $\sigma$.

\begin{definition}[Mistake-Known Multi-Label Littlestone Dimension] \label{def:known-multi-littlestone}
   Let $\cT$ be a complete multi-label tree with depth $d$. We say $\cT$ is shattered by $\cH$ if, for any path $\sigma = (\hat y_1, ..., \hat y_d)$ from root to some leaf, there exists a $h_\sigma \in \cH$ such that 
   for each $t \in [d]$, we have 
   \begin{enumerate}[label = (\roman*)]
   \item \emph{Consistency:} 
        $\nu_y(\sigma_{\le t}) \in h_\sigma\bigl(\nu_x(\sigma_{< t})\bigr)$,
    \item \emph{Mistake requirement:} $\hat y_{t} \notin h_\sigma(\nu_{x}(\sigma_{< t}))$.
\end{enumerate} 
   The Littlestone dimension $\LDK$ of $\cH$ in the mistake-known online setting is then defined as the depth of the largest shattered tree. 
\end{definition}

\begin{definition}[Mistake-Unknown Multi-Label Littlestone Dimension]
\label{def:unknown-multi-littlestone}
Let $\cT$ be a multi-label tree and let 
$\mof : \cH \to \bN \cup \{0\}$ be a \emph{mistake offset function}.
We say that $\cT$ is \emph{$d$-shattered} by $(\cH, \mof)$ if for every root-to-leaf path
$\sigma = (\hat y_1, \ldots, \hat y_T)$ in $\cT$, there exists a hypothesis
$h_\sigma \in \cH$ such that: 
\begin{enumerate}[label=(\roman*)]
    \item \emph{Consistency:} for all $t \in [T],\,
        \nu_y(\sigma_{\le t}) \in h_\sigma\bigl(\nu_x(\sigma_{<t})\bigr)$,
    \item \emph{Cumulative mistake requirement:}
    \(
        \sum_{t=1}^{T} \id
        \left[\hat y_{t} \notin h_\sigma(\nu_x(\sigma_{<t})) \right] + \mof(h_\sigma)
        \;\ge\; d .
    \)
\end{enumerate}
Specifically for any tree $\cT$ with depth $0$, we say $\cT$ is $d$-shattered by $(\cH, \mof)$ if there exists $h_\emptyset \in \cH$ such that $\mof(h_\emptyset) \geq d$.
The \emph{mistake-unknown multi-label Littlestone dimension} $\LDU(\cH, \mof)$
is the largest integer $d \in \bN$ such that there exists a $d$-shattered tree
by $(\cH, \mof)$. Moreover, we denote  $\LDU(\cH) := \LDU(\cH, {\mathbf 0})$, where ${\mathbf 0}(h) = 0$ for all $h \in \cH$.
\end{definition}

Intuitively, $\mof(h)$ accounts for the number of mistakes already incurred by
hypothesis $h$, reducing the remaining mistake budget available along any path.

\begin{claim} \label{claim:sup_offset}
    For any multi-label hypothesis class $\cH$ and any offset function $\mof: \cH \rightarrow \bN$, we have $\LDU(\cH,\mof) \geq \sup_{h \in \cH} \mof(h)$.
\end{claim}
\begin{proof}
    Take the tree $\cT$ of depth $0$ that only has a root, which is also a leaf. It is obvious that $h^{\star}$ with $\mof(h^{\star})  = \sup_{h\in \cH} \mof(h)$ satisfies property (ii) of Definition~\ref{def:unknown-multi-littlestone} for the path $\sigma= \emptyset$. This concludes that $(\cH,\mof)$ will $\sup_{h\in \cH} \mof(h)$-shatter $T$.
\end{proof}
The following theorem characterizes when all three multi-label online dimensions collapse to zero. The proof is deferred to Appendix~\ref{pf_LD_zero}.

\begin{restatable}{theorem}{LDzero} \label{thm:LD-0-1}
     For any multi-label hypothesis class $\cH$, we have 
        \begin{gather*}
        \LDS(\cH) = 0 \; \Leftrightarrow \; \LDK(\cH) = 0 \; \Leftrightarrow \;   \LDU(\cH) = 0 \; \Leftrightarrow \; \forall x\in \cX: \bigcap_{h \in \cH} h(x) \neq \emptyset. 
        \end{gather*}
\end{restatable}

\ifthenelse{\boolean{arxivsubmission}}{\section{A Characterization of Realizable Online Multi-Label Learning}\label{sec:real}}{\subsection{A Characterization of Realizable Online Multi-Label Learning}\label{sec:real}}

The following theorem characterizes the mistake bound of deterministic learners in the realizable online multi-label learning problem and under different feedback models, using the Littlestone dimensions defined in Section~\ref{sec:LDs}. 
We define a Standard Optimal Algorithm (SOA) for each setup and show they achieve optimal mistake bounds. The proof of Theorem~\ref{thm:real-char} and the description of SOAs appear in \ifthenelse{\boolean{arxivsubmission}}{the remainder of this section}{Appendix~\ref{pf_realizable_online}}.

\begin{restatable}{theorem}{RealChar}
\label{thm:real-char}
    For any multi-label hypothesis class $\cH$, and any $T \in \bN$ we have the following characterizations of the optimal mistake bound for the three feedback models:
    \usetagform{nowidth}
    \[
    \begin{aligned}
        &\text{Mistake-unknown feedback:} &&\inf_{\cA} \rM^{\unknown}_\cA(T, \cH)   = \min(\LDU(\cH), T). \\
        &\text{Mistake-known feedback:} && \inf_{\cA} \rM^{\known}_\cA(T, \cH)  = \min(\LDK(\cH), T). \\
       &\text{Set-valued feedback:} &&\inf_{\cA} \rM^{\set}_\cA(T, \cH) = \min(\LDS(\cH), T).  \\
        \end{aligned}
   \]
\end{restatable} 

\citet{joshi2025learning} showed that for any finite concept class $\cH$, there exists a \fmedit{deterministic} algorithm with a mistake bound of at most $\log_2 |\cH|$ \fmedit{in the mistake-unknown setup}. Combined with Theorem~\ref{thm:real-char}, this immediately yields the following.

\begin{corollary}\label{corollary:LD_finite} 
     For a finite multi-label hypothesis class $\cH$ we have $\LDS(\cH) \leq \LDK(\cH) \leq \LDU(\cH) \leq \log_2 |\cH|$.
\end{corollary}

\ifthenelse{\boolean{arxivsubmission}}{{\bf Proof of Theorem~\ref{thm:real-char}.}
\label{pf_realizable_online}
\begin{algorithm}[tb!]
\caption{Mistake Unknown Multi-label Standard Optimal Algorithm (SOA)}
\begin{algorithmic}[1] 
\Require Multi-label hypothesis class $\cH$ and $T \in \bN$.
\State Initialize $\azreplace{V_0}{V_1} = \cH$, $\azreplace{\mof_0}{\mof_1} \equiv 0$.
\For{$t = 1$ \textbf{to} $T$}
    \State Receive $x_t$ 
    \State For $r \in \cY$, define $V^{(r)}_t := \{h \in V_t: r \in h(x_t)\}$
    \State For every $\hat r \in \cY$ and $h \in \cH$ define $\mof^{(\hat r)}_t(h) := \mof_t(h) + \id [\hat r \notin h(x_t)]$

    \State Predict $\hat y_t := \argmin{\hat{r} \in \cY} \max_{r\in \cY} \LDU(V^{(r)}_t, \mof^{(\hat{r})}_t)$   
    
    \State Receive label $y_t$
    \State Set $V_{t + 1} = V^{(y_t)}_t$, and $\mof_{t + 1} =\mof^{(\hat y_t)}_t$.
   
\EndFor
\end{algorithmic}
\label{alg:mistake-unknown-SOA}
\end{algorithm}
\begin{algorithm} [tb!]
\caption{Mistake Known Multi-label Standard Optimal Algorithm (SOA)}
\begin{algorithmic}[1] 
\Require Multi-label hypothesis class $\cH$ and $T \in \bN$.
\State Initialize $\azreplace{V_0}{V_1} = \cH$.
\For{$t = 1$ \textbf{to} $T$}
    \State Receive $x_t$
    \State For $r, \hat r \in \cY$, define $V^{(r, \hat r, 1)}_t := \{h \in V_t: r \in h(x_t), \hat r \notin h(x_t) \}$ and $V^{(r, \hat r, 0)}_t := \{h \in V_t: r \in h(x_t), \hat r \in h(x_t) \}$
    \If{$\LDK(V_t) > 0$}
        \State Predict $\hat y_t := \argmin{\hat{r} \in \cY} \max_{r\in \cY} \LDK(V^{(r, \hat r, 1)}_t)$
        
    \Else
        \State Predict any $\hat y_t \in \bigcap_{h \in V_t} h(x_t)$ (exists due to Theorem~\ref{thm:LD-0-1})
    \EndIf
     \State Receive label $y_t$ and $b_t$ indicating whether $\hat y_t$ was mistake
    \State Set $V_{t + 1} = V^{(y_t, \hat y_t, b_t)}_t$
\EndFor
\end{algorithmic}
\label{alg:mistake-known-SOA}
\end{algorithm}
\begin{algorithm} 
\caption{Set-valued Standard Optimal Algorithm (SOA)}
\begin{algorithmic}[1] 
\Require Multi-label hypothesis class $\cH$ and $T \in \bN$.
\State Initialize $\azreplace{V_0}{V_1} = \cH$.
\For{$t = 1$ \textbf{to} $T$}
    \State Receive $x_t$ 
    \State For $S \subseteq \cY$ and $\hat r \in \cY$ define $V^{(S, \hat r)}_t := \{h \in V_t:  h(x_t) = S, \hat r \notin h(x_t) \}$
    \If{$\LDS(V_t) > 0$}
        \State Predict $\hat y_t := \argmin{\hat{r} \in \cY} \max_{S \subseteq \cY} \LDS(V^{(S, \hat r)}_t)$
        
    \Else
        \State Predict any $\hat y_t \in \bigcap_{h \in V_t} h(x_t)$ (exists due to Theorem~\ref{thm:LD-0-1})
    \EndIf
    \State Receive set feedback $S_t$ indicating whether $\hat y_t$ was mistake
    \State Set $V_{t + 1} = V^{(S_t, \hat y_t)}_t$
\EndFor
\end{algorithmic}
\label{alg:set-SOA}
\end{algorithm}
We prove the characterization for each feedback model in the following.

    \textbf{Mistake-unknown characterization:} We first prove that for any $T \geq \azreplace{\LDK(\cH)}{\LDU(\cH)}$ and any algorithm $\cA$ we have $\rM^{\unknown}_\cA(T, \cH) \geq \LDU(\cH)$. Let $\cT$ be a multi-label tree that is $\LDU(\cH)$-shattered by $\cH$ and let $x_1$ denote its root. For every $t \in  [\LDU(\cH)]$, denote $\hat y_{t} := p_{\cA}(t)$, $y_{t} = \nu_y (\hat y_1, \ldots, \hat y_{t})$, and $x_{t+1} = \nu_x (\hat y_1, \ldots, \hat y_{t})$. Let $h^{\star} = h_{\sigma^{\star}}$ for $\sigma^{\star} = (\hat y_1, \ldots, \hat y_{\LDK(\cH)})$. From property (i) of Definition~\ref{def:unknown-multi-littlestone} we have $y_{t} \in h^{\star}(x_t)$. Moreover, due to property (ii) $\sum_{t = 1}^{\LDU(\cH)} \id [ \hat y_{t} \notin h^{\star}(x_t) ] \geq \LDU(\cH)$. This concludes that $\cA$ makes at least $\LDU(\cH)$ mistakes and the claim follows.

Next we show that for the learner $\cA$ defined in Algorithm~\ref{alg:mistake-unknown-SOA} and any $T \in \bN$ we have $\rM^{\unknown}_\cA(T, \cH) \leq \LDU(\cH)$. We prove this by showing that $\LDU(V_t, \mof_t)$ is non-increasing in $t$. Note that this finishes the proof since for $h^\star \in \cH$,
\[
\begin{aligned}
    \LDU(\cH) &= \LDU(\azreplace{V_0}{V_1}, \azreplace{\mof_0}{\mof_1}) \\
    &\geq \LDU(\azreplace{V_T}{V_{T+1}}, \azreplace{\mof_T}{\mof_{T+1}})  && (\text{since $\LDU(V_t, \mof_t)$ is non-increasing})\\
    & \geq \sup_{h \in \azreplace{V_T}{V_{T+1}}} \azreplace{\mof_T}{\mof_{T+1}}(h) && (\text{due to Claim~\ref{claim:sup_offset}})\\
    & \geq \azreplace{\mof_T}{\mof_{T+1}}(h^\star) \\
    & = \rM^{\unknown}_\cA(T, \cH),
\end{aligned}
\]
where the last equality follows from the fact that we can inductively prove 
\[
\azreplace{\mof_T}{\mof_{T+1}}(h^{\star}) = \azreplace{\mof_{T-1}}{\mof_T}(h^{\star}) +\id[\hat{y}_\azreplace{{T-1}}{T} \notin h^{\star}(x_\azreplace{{T-1}}{T})] = \sum_{t=1}^T \id[\hat{y}_{t} \notin h^{\star}(x_t)] = \sum_{t=1}^T \id[p_{\cA}(t) \notin h^{\star}(x_t)].
\]

We now prove that $\LDU(V_t, \mof_t)$ is non-increasing in $t$. For the sake of contradiction suppose $\LDU(V_{t + 1}, \mof_{t+1}) > \LDU(V_{t}, \mof_{t})$ for some $t \in [T]$. For every $\hat r \in \cY$ denote 
\[
y_{\mathsf{max}}(\hat r) := \argmax{r} \; \LDU(V^{(r)}_t, \mof^{(\hat{r})}_t).
\] 
Note that from the description of Algorithm~\ref{alg:mistake-unknown-SOA}, we have \azreplace{$\hat{y}_t = \arg\min_{\hat{r} \in \cY} y_{\mathsf{max}}(\hat r)$}{\[\hat{y}_t = \arg\min_{\hat{r} \in \cY} \LDU(V^{(y_{\mathsf{max}}(\hat r))}_t, \mof^{(\hat r)}_t).\]} Therefore, for all $\hat r \in \cY$, we have,
\[
\begin{aligned}
\LDU(V^{(y_{\mathsf{max}}(\hat r))}_t, \mof^{(\hat{r})}_t) & \geq \LDU(V^{(y_{\mathsf{max}}(\hat y_t))}_t, \mof^{(\hat{y}_t)}_t) \\
& \geq  \LDU(V^{(y_t)}_t, \mof^{(\hat{y}_t)}_t) \\
& = \LDU(V_{t + 1}, \mof_{t+1}). \\    
\end{aligned}
\]
\azedit{For any $\hat{r} \in \cY$,} we let $\cT_{\hat r}$ be a multi-label tree with $\LDU(V^{(y_{\mathsf{max}}(\hat r))}_t, \mof^{(\hat{r})}_t) \geq \LDU(V_{t + 1}, \mof_{t+1}) $ that is shattered by $\left(V^{(y_{\mathsf{max}}(\hat r))}_t, \mof^{(\hat r)}_t\right)$, where the label on the root is $y_{\mathsf{max}}(\hat r)$. 
Construct the multi-label tree $\overline{\cT}$ such that root of $\overline{\cT}$ is $x_t$, and each edge $\hat r \in \cY$ of the root is connected to $\cT_{\hat r}$. Then $\overline{\cT}$ is shattered by $(V_t, \mof_t)$. This indicates that $\LDU(V_t, \mof_t) \geq \LDU(V_{t + 1}, \mof_{t+1})$ which is a contradiction. 

    \textbf{Mistake-known characterization:} We first show that for any $T \geq \LDK(\cH)$ and any algorithm $\cA$ we have $\rM^{\known}_\cA(T, \cH) \geq \LDK(\cH)$. Find a complete multi-label tree $\cT$ of depth $\LDK(\cH)$ that is shattered by $\cH$. Let $x_1$ be the root of $\cT$. For every $t \in  [\LDK(\cH)]$, define $\hat y_t := p_\cA(t)$, $y_t = \nu_y (\hat y_1, ..., \hat y_t)$ and $x_{t+1} = \nu_x (\hat y_1, ..., \hat y_t)$. Let $h^\star = h_{\sigma^\star}$ for $\sigma^\star = (\hat y_1, \ldots, \hat y_{\LDK(\cH)})$. First note that due to property (i) of Definition~\ref{def:known-multi-littlestone} we have $y_t \in h^\star(x_t)$. Moreover due to property (ii), $\hat y_t \notin h^\star(x_t)$. This concludes  $\rM^{\known}_\cA(T, \cH) \geq \LDK(\cH)$. 

    Next we prove that for the learner $\cA$ defined in Algorithm~\ref{alg:mistake-known-SOA} and any $T \in \bN$, we have $\rM^{\known}_\cA(T, \cH) \leq \LDK(\cH)$. We prove this by showing for each $t$, if $\LDK(V_{t}) > 0$ and $\cA$ makes a mistake, then $\LDK(V_{t+1}) \leq \LDK(V_{t}) - 1$. Note that this concludes the claim since Theorem~\ref{thm:LD-0-1} ensures that whenever $\LDK(V_t) = 0$, $\bigcap_{h \in V_t} h(x_t) \neq \emptyset$ and $\cA$ makes no more mistakes. 

    For the sake of contradiction suppose $\cA$ makes a mistake at $t$ and $\LDK(V_{t + 1}) = \LDK(V_{t})$. For every $\hat r \in \cY$ denote $y_{\mathsf{max}}(\hat r) := \argmax{r} \; \LDK(V^{(r, \hat r, 1)}_t)$. Then, noting that $m_t=1$, we have
    \[
    \begin{aligned}
        \LDK(V^{(y_{\mathsf{max}}(\hat r), \hat r, 1)}_t) &\geq \LDK(V^{(y_{\mathsf{max}}(\hat{y}_t), \hat{y}_t, 1)}_t) \geq \LDK(V^{(y_t, \hat{y}_t, 1)}_t) \\
        & = \LDK(V_{t + 1}) = \LDK(V_{t}).
    \end{aligned}
    \] 
    \azedit{For any $\hat{r}\in\cY$}, let $\cT_{\hat r}$ be a complete multi-label tree of depth $\LDK(V_t)$ that is shattered by $V^{(y_{\mathsf{max}}(\hat r), \hat r, 1)}_t$, where the label on the root is $y_{\mathsf{max}}(\hat r)$. Construct the complete multi-label tree $\overline{\cT}$ such that root of $\overline{\cT}$ is $x_t$, and each edge $\hat r \in \cY$ of the root is connected to $\cT_{\hat r}$. Clearly $\overline{\cT}$ is shattered by $\bigcup_{\hat r} V^{(y_{\mathsf{max}}(\hat r), \hat r, 1)}_t $, which implies that it is also shattered by $V_t 
    \supseteq \bigcup_{\hat r} V^{(y_{\mathsf{max}}(\hat r), \hat r, 1)}_t $. Hence, $\LDK(V_t) \geq \LDK(V_{t}) + 1$ which is a contradiction.
    
    \textbf{Set-valued characterization:} We first show that for any $T \geq \LDS(\cH)$ and any algorithm $\cA$ we have $\rM^{\set}_\cA(T, \cH) \geq \LDS(\cH)$. Find a set-valued tree $\cT$ of depth $\LDS(\cH)$ that is shattered by $\cH$. Let $x_1$ be the root of $\cT$. For every $t \in  [\LDS(\cH)]$, define $\hat y_t := p_\cA(t)$, $S_t = \nu_S (\hat y_1, ..., \hat y_t)$ and $x_{t+1} = \nu_x (\hat y_1, ..., \hat y_t)$. Let $h^\star = h_{\sigma^\star}$ for $\sigma^\star = (\hat y_1, \ldots, \hat y_{\LDS(\cH)})$. First note that due to property (i) of Definition~\ref{def:set-multi-littlestone} we have $S_t = h^\star(x_t)$. Moreover due to property (ii), $\hat y_t \notin h^\star(x_t)$. This concludes  $\rM^{\set}_\cA(T, \cH) \geq \LDS(\cH)$. 

    Next we prove that for the learner $\cA$ defined in Algorithm~\ref{alg:set-SOA} and any $T \in \bN$, we have $\rM^{\set}_\cA(T, \cH) \leq \LDS(\cH)$. We prove this by showing at each $t$, if $\LDS(V_{t}) > 0$ and $\cA$ makes a mistake, then $\LDS(V_{t+1}) \leq \LDS(V_{t}) - 1$. Note that this concludes the claim since Theorem~\ref{thm:LD-0-1} ensures that whenever $\LDS(V_t) = 0$, we have $\bigcap_{h \in V_t} h(x_t) \neq \emptyset$ and $\cA$ makes no more mistakes. 

    For the sake of contradiction suppose $\cA$ makes a mistake at $t$ and $\LDS(V_{t + 1}) = \LDS(V_{t})$. For every $\hat r \in \cY$ denote $S_{\mathsf{max}}(\hat r) := \argmax{S \subseteq \cY} \; \LDS(V^{(S, \hat r)}_t)$
    
    \[\LDS(V^{(S_{\mathsf{max}}(\hat r), \hat r)}_t) \geq \LDS(V^{(S_{\mathsf{max}}(\hat{y}_t), \hat{y}_t)}_t) \geq \LDS(V^{(S_t, \hat{y}_t)}_t)  = \LDS(V_{t + 1}) = \LDS(V_{t}).
    \] Let $\cT_{\hat r}$ be a complete multi-label tree of depth $\LDS(V_t)$ that is shattered by $V^{(S_{\mathsf{max}}(\hat r), \hat r)}_t$, where the set of labels on the root is $S_{\mathsf{max}}(\hat r)$. Construct the complete multi-label tree $\overline{\cT}$ such that root of $\overline{\cT}$ is $x_t$, and each edge $\hat r \in \cY$ of the root is connected to $\cT_{\hat r}$. Clearly $\overline{\cT}$ is shattered by $\bigcup_{\hat r} V^{(S_{\mathsf{max}}, \hat r)}_t $, which implies that it is also shattered by $V_t 
    \supseteq \bigcup_{\hat r} V^{(S_{\mathsf{max}}(\hat r), \hat r)}_t $. Hence, $\LDS(V_t) \geq \LDS(V_{t}) + 1$ which is a contradiction.     \hfill \qed}{}

\section{Agnostic Online Multi-label Learning} \label{sec:agnostic}
 Our first result in this section shows a separation between realizable and agnostic online learnability in the mistake-unknown setting by exhibiting a simple hypothesis class with $\LDU$ of $1$ but a regret that always grows linearly in $T$.

In contrast, we prove for the same class that the regret under the set-valued model is only constant. This leads to the question of whether agnostic online learning is possible in the mistake-known setting which lies between the two in terms of the feedback.

 Interestingly, we find that although the feedback is partial in the mistake-known setting, the regret is always sub-linear and is upper bounded by $O(\sqrt{\LDK(\cH)|\cY|T\log(T)})$ (Theorem~\ref{thm:agnostic_mistake_unknown}).
 This shows a strong gap between the mistake-unknown and mistake-known feedback models in terms of whether the corresponding Littlestone dimensions yield sub-linear regret bounds in the agnostic setting.
Finally, we prove that regret in the set-valued setting is bounded by $O(\sqrt{\LDS(\cH)T\log(T))}$ and independent of $|\cY|$ (Theorem~\ref{thm:agnostic_set_valued}), thus, holding for infinite label spaces too.

\subsection{A Gentle Start: An Example Hypothesis Class} \label{sec:agnostic-H3}
We start by presenting a simple hypothesis class over a singleton domain with label set $\{1,2,3\}$ for which any algorithm suffers linear regret in the mistake-unknown setting, while the regret for the same class can be made constant in the set-valued setting.  We give a high-level description of 
the strategy that is used by the adversary to force a linear regret. The adversary selects $y_t$ uniformly at random from $\cY$ for all $t\in[T]$. Since there is only a single instance, in the mistake-unknown setting the learner is only a function of $(y_i)_{i<t}$ and is independent of $(S_i)_{i<t}$, as they are not revealed. Therefore, by using the label chosen most often by the learner over the $T$ rounds, the adversary can construct $S_t$ so as to force a regret linear in $T$.

The same strategy, however, does not work in the mistake-known and set-valued setups, as the prediction of the learner at time $t$ also depends on the choice of $S_i$ for $i<t$. From Theorem~\ref{thm:agnostic_mistake_unknown}, we know that there is an algorithm with sub-linear regret for this class in the mistake-known setting. More interestingly, we give an algorithm (Algorithm~\ref{alg:set-goodH} in Appendix~\ref{pf_H3}) that enjoys a constant regret
of at most $\log_2(3)$ under the set-valued setting. This algorithm is inspired by Algorithm~1 in
\citet{joshi2025learning} for mistake-unknown online learning over finite
hypothesis classes.

Algorithm~\ref{alg:set-goodH} assigns each hypothesis $h \in \cH$ a weight $w^{(t)}(h)$, initialized to $1$. At each round $t$, it predicts a label $y$ with probability proportional to the total weight of hypotheses containing $y$. It then updates the weights as follows: if a mistake occurs, it doubles the weights of hypotheses consistent with $S_t$, i.e., those relative to which regret increases, and keeps the weight of others unchanged; 
otherwise, it halves the weights of those that are not consistent with $S_t$, i.e., hypotheses relative to which regret decreases, and again, keeps the remaining weights unchanged.

We can then show that (i) the expectation of the sum of weights over all hypotheses is non-increasing, and (ii) that the expected sum of weights at round $T$ must have scaled exponentially with the regret.

The key differences between Algorithm~\ref{alg:set-goodH} and that of \citet{joshi2025learning} are
threefold: (i) our algorithm incorporates full knowledge of the true label set $S_t$;
(ii) in contrast to \azreplace{\citet{joshi2025learning}, our method}{their method, ours} is probabilistic; and
(iii) while they show the total sum of weights is non-increasing,
we show its expectation is non-increasing. The formal proof appears in Appendix~\ref{pf_H3}.

\begin{restatable}{theorem}{HThree}   
Consider the hypothesis class $\mathcal{H}_3 = \{h_1, h_2, h_3\}$ over domain $\cX = \{x\}$ and label set $\cY = \{1, 2, 3\}$, where $h_1(x) = \{2, 3\}$, $ h_2(x) = \{1, 3\}$, and $h_3(x) = \{1, 2\}$. We have $\LDU(\cH_3) = \LDS(\cH_3) = 1$, whereas
 \[
    \inf_\cA \rR^{\unknown}_{\cA}(T, \cH_3) \geq \frac{T}{9},\quad \inf_{\cA} \rR^{\set}_{\cA}(T, \cH_3) \leq \log_2 (3).
\]
\end{restatable} \label{thm:H3}

\subsection{Agnostic Online Learning Under Mistake-Known Feedback Model}\label{sec:regret_mistake_known}

In this section we show that the regret in the mistake-known feedback model is always sub-linear by constructing a learner that attains the regret guarantee stated in Theorem~\ref{thm:agnostic_mistake_unknown}. Existing results for multi-class agnostic online learning imply a $\Omega(\sqrt{T \,\LDK })$ lower bound for hypothesis classes with single correct label \cite{daniely2015multiclass}.

\begin{theorem}\label{thm:agnostic_mistake_unknown}
    For any multi-label hypothesis class $\cH$, and any $T \geq \LDK(\cH)$,
    \[
    \inf_{\cA} \rR^{\known}_{\cA}(\cH,T) \leq O\left(\sqrt{\LDK(\cH)|\cY|T\log(T)}\right).
    \]
\end{theorem}
 To prove this theorem we will use the idea of reduction to realizable learning \citep{hanneke2023multiclass}. For any subset $I\subseteq [T]$ of size at most $\LDK(\cH)$, we create a mistake-known SOA (Algorithm~\ref{alg:mistake-known-SOA} in Appendix~\ref{pf_realizable_online}) that makes a prediction solely based on the instances in $I$ observed so far, i.e., the memory of SOA will only contain instances from $I$. We aggregate the SOAs over all such subsets to create a set of experts that advise prediction.
 Unlike the setting of prediction with expert advice, the learner knows an expert made a mistake only if it follows its advice.  
However, executing the mistake-known SOA requires knowing whether its predictions were mistakes in every past round, which is unavailable in rounds where its recommended label is not predicted by the learner.

 The above may render the suggested approach impossible. The good news is that when the problem is reduced to realizable learning, we only rely on the fact that there exists a subset $I^*$ of size at most $\LDK$ from the realizable subset of data with the following property: $I^*$ exactly equals the set of mistakes that the SOA defined on $I^*$ will make on the realizable subset of data. Therefore, this SOA should be executed with a mistake bit of $1$ for every round in its history. In other words, executing every expert with mistake bit of $1$ guarantees the existence of one with ``small'' number of mistakes.

Moreover, since information in the feedback is partial we cannot apply the classic tool from the setting of prediction with expert advice. Hence, we borrow tools from the bandits literature.

  We describe the bandit setup we use in our proof. The game plays between a learner $\cA$ and an adversary in $T$ rounds. We have a set $\{1,\ldots,K\}$ of $K\in \bN$ actions and a set $\cE =\{E_1,\ldots,E_N\}$ of $N\in \bN$ experts. At round $t$, the adversary chooses an array of rewards $g(t) = (g^{(1)}(t),\ldots,g^{(K)}(t)) \in [0,1]^K$ for the actions but does not reveal it to the learner. Each expert $E_i$ (potentially based on the history) outputs a probability distribution $p_i(t) = (p_i^{(1)},\ldots,p_i^{(K)}) \in [0,1]^K$  with $\sum_{j\in [K]}p_i^{(j)}=1$ over the set of actions. The learner $\cA$ at each round $t$ observes the distributions $p_1(t),\ldots,p_N(t)$ of the experts based on which it computes a probability distribution $p_{\cA}(t) = (p_{\cA}^{(1)},\ldots,p_{\cA}^{(K)}) \in [0,1]^K$  with $\sum_{j\in [K]}p_{\cA}^{(j)}=1$ and draws an action $\hat{y}\sim p_{\cA}(t)$. The adversary then only reveals the reward $g^{(\hat{y})}(t)$.  Define by $E_{U}$ the uniform (randomized) expert such that its probability distribution $p_{U}(t)$ satisfies $ p^{(j)}_{U}(t)= 1/|K|$ for all $t\in T$ and $j\in K$. \citet{auer2002nonstochastic} introduce the $\EXPfour$ algorithm for the setting of non-stochastic bandits with prediction advice with guarantees summarized in the following lemma. 
    \begin{lemma}[Theorem~7.1 in \cite{auer2002nonstochastic}]\label{lemma:exp4}
        Let $K,N\in \bN, \gamma \in (0,1]$. Let $\{1,\ldots,K\}$ be a set of actions and $\cE =\{E_1,\ldots,E_N\}$ a set of $N$ experts with $E_U \in \cE$. The algorithm $\EXPfour$ in \cite{auer2002nonstochastic}) satisfies
        \[
         \max_{i\in [K]}\sum_{i=1}^{T}\langle p_i(t),g(t)\rangle -\sum_{i=1}^{T}\langle p_{\EXPfour}(t),g(t)\rangle\leq (e-1)\gamma \max_{i\in [K]}\sum_{i=1}^{T}\langle p_i(t),g(t)\rangle + \frac{K\ln(N)}{\gamma}.
        \]
    \end{lemma}
\begin{proofof}{Theorem~\ref{thm:agnostic_mistake_unknown}}
     Let $X=(x_1,\ldots,x_T)$, $Y = (y_1,\ldots,y_T)$, and $\mathcal S=(S_1,\ldots,S_T)$ be the online sequence. For any $n\in\bN$ denote by $\mathbf{1}_n$ the all-ones sequence of size $n$. Let $\MKSOA$ be the Mistake Known SOA learner defined in Algorithm~\ref{alg:mistake-known-SOA}. We will define a set $\cE$ of experts as follows. 
    
    Let $\cI = \{I \subseteq [T]: |I| \leq \LDK(\cH)\}$. For each set of indices $I \in \cI$, we define an expert $E_{I}$ that in round $t$ predicts as $E_{I}(X_{<t},Y_{<t})(x_t) = \MKSOA(X_{I_{<t}},Y_{I_{<t}},\mathbf{1}_{t-1})(x_t)$, where $I_{<t} = \{i \in I: i<t\}$ is the subset of indices in $I$ that are smaller than $t$. In other words, at round $t$, the expert $E_I$ only has in its memory the indices in $I$ that appeared so far. Let $\cE = \{E_I: I \in \cI\}\cup \{E_U\}$, where $E_U$ is the uniform expert that always outputs a uniform distribution over the label set $\cY$. 
    
    Denote $E_{I}^{(t)}:= E_{I}(X_{<t},Y_{<t})(x_t)$. We will show that there exists an expert $E_{I^*}$ such that 
    \begin{equation} \label{eq:agnostic_mistake_unknown-expert-Istar}
        \sum_{t=1}^T \id [E_{I^*}^{(t)} \notin S_t] \leq \LDK(\cH) + \min_{h\in\cH}\sum_{t=1}^T \id[h(x_t) \neq S_t].
    \end{equation}

    We first show how this concludes the claimed regret upper bound. We claim this problem can be reduced a non-stochastic bandits with expert advice problem. The adversary at each time $t$, will (potentially randomly and based on the history) select a tuple $(x_t,y_t,S_t)$. Each expert in $\cE$ makes a prediction and the learner makes a prediction $\hat{y}_t$ based on these recommendations. The adversary then reveals the reward of choosing $\hat{y}_t$, i.e., it feeds back $g^{\hat{y}}(t) := \id[\hat{y}_t \in S_t]$. Note that $S_t$ can be viewed as defining a reward vector $g(t)\in\{0,1\}^{|\cY|}$ with $g^{\hat{y}}(t) := \id[\hat{y}\in S_t]$. Observe crucially that the prediction of experts at time $t$ depend only on the past history $(X_{<t},Y_{<t})$ and the current instance $x_t$: the experts are SOAs that are only fed back (a subset of) the past history, with mistake bits being equal to $1$, and the current instance.

    We know $\cE$ is a set of $\sum_{i\leq \LDK(\cH)} {T \choose i} \leq \left(eT/\LDK(\cH)\right)^{\LDK(\cH)}$ experts over the actions space $\cY$. The algorithm $\cA$ then uses the $\EXPfour$ algorithm~\citep{auer2002nonstochastic} on these experts. Note that if $T \leq \frac{|\cY|\ln(|\cE|)}{e-1}$, we obviously have a mistake (and, thus, regret) bound, of $T\leq O(\sqrt{\LDK(\cH)|\cY|T\log(T)})$. Otherwise, by setting $\gamma = \sqrt{\frac{|\cY|\ln(|\cE|)}{(e-1)T}}$ in Lemma~\ref{lemma:exp4} we immediately get
    \[
    \max_{E\in \cE}\sum_{i=1}^{T}\langle E_I^{(t)},g(t)\rangle - \sum_{i=1}^{T}\langle p_{\cA}(t),g(t)\rangle   \leq 2\sqrt{(e-1)|\cY|T\LDK(\cH)\ln\left(\frac{eT}{\LDK(\cH)}\right)}.
    \]
    On the other hand note that 
    \[
    \begin{aligned}
             \max_{E\in \cE}\sum_{i=1}^{T}\langle E_I^{(t)},g(t)\rangle - \sum_{i=1}^{T}\langle p_{\cA}(t),g(t)\rangle & =    \max_{E\in \cE}\sum_{i=1}^{T}\id[ E_I^{(t)}\in S_t] - \expect{\sum_{i=1}^{T} \id [p_{\cA}(t) \in S_t]}\\
            &  =  \expect{\sum_{i=1}^{T} \id [p_{\cA}(t) \notin S_t]} - \min_{E\in \cE}\sum_{i=1}^{T}\id[ E_I^{(t)}\notin S_t].
    \end{aligned}
    \]
   Therefore, we get that 
   \[
   \begin{aligned}
       \expect{\sum_{i=1}^{T} \id [p_{\cA}(t) \notin S_t]} - \min_{h\in\cH}\sum_{t=1}^T \id[h(x_t) \neq S_t]\leq  \LDK(\cH) + 4\sqrt{\LDK(\cH)|\cY|T\log(T)},
   \end{aligned}
   \]
   as desired. 
    
We turn to proving the existence of the expert claimed in \eqref{eq:agnostic_mistake_unknown-expert-Istar}. Let $h^*\in \arg\min_{h\in \cH}\sum_{t=1}^T \id[h(x_t) \neq S_t]$ and $R\subseteq[T]$ be the subset of indices on which $h^*$ is correct. Clearly, $|R|\geq T - \min_{h\in \cH}\sum_{t=1}^T \id[h(x_t) \neq S_t] $. Consider the following thought experiment. For any $t\in[|R|]$ let $R(t)\in[T]$ denote the index at position $t$ in $R$. If $\MKSOA(\emptyset,\emptyset,0)(x_{R(1)}) \in S_{R(1)}$, then discard $x_{R(1)},y_{R(1)}$ and move to $R_2$. 
Repeat the same process until we get to the first instance $i^*_1$ such that   $\MKSOA(\emptyset,\emptyset,0)(x_{R(i^*_1)})\notin S_{R(i^*_1)}$.  
We keep $x_{R(i^*_1)},y_{R(i^*_1)}$ in history.
We continue this process for $|R|$ many times and will keep or discard the instances based on the correctness of the prediction.

Formally, let $I^*(0)=\emptyset$, $n_0=0$. For $t\in [|R|]$, 
if $\MKSOA((X_{I^*(t-1)},Y_{I^*(t-1)}),\mathbf{1}_{n_{t-1}})(x_{R(t)}) \in S_{R(t)}$, discard $x_{R(t)},y_{R(t)}$, let $n_t=n_{t-1}$ and $I^*(t)= I^*(t-1)$. Otherwise, let $n_t = n_{t-1}+1$ and $I^*(t)=I^*(t-1)\cup R(t)$. 
Repeat this process until $t=|R|$ and let $I^* = I(|R|)$. Observe that this SOA will actually make a mistake on any $i^*\in I^*$ and since $R$ is realizable by $h^*$ we must have $|I^*|\leq \LDK(\cH)$. Therefore, $I^*\in \cI$ and hence $E_{I^*} \in \cE$. In other words, the SOA that makes a mistake on every index in $I^*$ and is correct on any instance in $R \setminus I^*$ is included in $\cE$. For this expert we have  $\sum_{i=1}^{T} \id [E_{I^*}^{(t)} \notin S_t] \leq |I^*|+ |[T]\setminus R| \leq \LDK(\cH) + \min_{h\in \cH}\sum_{t=1}^T \id[h(x_t) \neq S_t]$, as desired.
\end{proofof}

\subsection{Agnostic Online Learning Under Set-valued Feedback Model}\label{section:agnostic_set_valued}
In this section we focus on 
regret in the set-valued agnostic online setting. We first give the characteristics of a family of hypothesis classes 
with regret bound that is
at most logarithmic in the size of the 
class. Our next result is a regret upper bound of $O(\sqrt{\LDS(\cH)T\log(T)})$ for any class $\cH$, which compared to Theorem~\ref{thm:agnostic_mistake_unknown} has no dependence on $|\cY|$ and works even for infinite label spaces.

{\bf A family of hypothesis classes with constant regret.} In Section~\ref{sec:agnostic-H3}, we established a constant set-valued regret bound for a
simple multi-label hypothesis class. Theorem~\ref{thm:set-valued-const} extends this result
by showing that the same algorithm achieves constant set-valued regret for a large family of
multi-label hypothesis classes. 
The proof follows the same ideas as in
Theorem~\ref{thm:H3}. In particular, it suffices to show that the total sum of
weights remains non-increasing for this broader family of hypothesis classes. The proof of Theorem~\ref{thm:set-valued-const} is deferred to Appendix~\ref{pf_agnostic_online}.

\begin{restatable}{theorem}{SetValConst} \label{thm:set-valued-const}
Let $\cH$ be a finite class of multi-label hypotheses such that
\[
\forall h,h' \in \cH \text{ and } \forall x \in \cX:\quad
2|h'(x)\setminus h(x)| \leq |h(x)|.
\]
There exists an algorithm $\cA$ such that
\[
\rR^{\set}_{\cA}(T, \cH) \leq \log_2(|\cH|).
\]
\end{restatable} 

{\bf Sub-linear regret bound independent of label set size.} In section~\ref{sec:regret_mistake_known} we showed a regret bound that is sub-linear but scales with $O(\sqrt{|\cY|})$. The following theorem highlights that in the set-valued setting the regret can be made independent of label set size. The proof of the following theorem appears in Appendix~\ref{section:pf_agnostic_set_valued} and follows techniques similar to those used in \citet{raman2024online}, which in turn build on earlier proof techniques from \citet{hanneke2023multiclass}.

\begin{restatable}{theorem}{SetValAgnostic} \label{thm:agnostic_set_valued}
    For any multi-label hypothesis class $\cH$, and any $T \geq \LDS(\cH)$,
    \[
    \inf_{\cA} \rR^{\set}_{\cA}(\cH,T) \leq O\left(\sqrt{\LDS(\cH)T\log(T)}\right).
    \]
\end{restatable}

\ifthenelse{\boolean{arxivsubmission}}{\section{Learning from Multiple Correct Answers in the Batch Setup}\label{sec:batch}

\azdelete{We have so far studied the problem of learning from multiple correct answers in the online setting and under various feedback models.} In this section, we show our online results imply learnability in different PAC settings too. We use \azdelete{standard} Online-to-Batch conversion techniques and show the relevant combinatorial parameters defined in Section~\ref{sec:LDs} can control the sample complexity of learning from multiple answers in the batch setting. We start with the main lemma that we exploit to turn our online learners into batch learners, the proof of which can be found in Appendix~\ref{pf_online_to_batch}. We then give formal definition of different Batch settings we consider. 
\begin{lemma}\label{lemma:A_o2b}
    Let  $\cA: \left(\cX \times \cZ \right)^* \rightarrow \Pi(\cY)^\cX$ and $\cD$ be a distribution over $\cX \times \cZ$. 
    For any $U := \left((x_1, z_1), ..., (x_m, z_m)\right) \in (\cX \times \cZ)^*$ define  
    $\cA_{\mathrm{o2b}}: \left(\cX \times \mathcal{Z}\right)^* \rightarrow \Pi(\cY)^\cX$ as
    \[
     \forall x \in \cX: \cA_{o2b}(U)(x) \sim \Unif \big(f_1(x),\ldots,f_{m}(x) \big), \; \text{where } \; \forall t \in [m]: f_t(x) := \cA(U_{<t})(x).
    \]
    Then for any $\varepsilon, \delta \in (0,1] \times (0, 1]$ and $h^*:\cX \rightarrow 2^{\cY}$,  with probability at least $1 - \delta$ over an i.i.d~generated sample $U := \left((x_1, z_1), \ldots, (x_m, z_m)\right)$ from $\cD$ we have
    \[
     \probs{x\sim \cD_{\cX}}{\cA_{\mathrm{o2b}}(U) \notin h^*(x)} \leq \frac{1+ \sum_{t = 1}^{m} \ell_t+12 \log \left(\frac{2 \log m}{\delta}\right)}{m},
    \]
    where $\ell_t := \id \left[f_t\left(x_t\right) \notin h^*\left(x_t\right)\right]$.
\end{lemma}

We now give the definition of \fmreplace{PAC learning from multiple correct answers}{realizable multi-label learning}. \azedit{We say a distribution $\cD$ over $\cX \times \cY$ is realizable by a function $h:\cX \mapsto 2^{\cY}$, if $ \probs{(x,y)\sim \cD}{y \in h(x)}=1$.}

\begin{definition}[Realizable  Multi-Label Learning, Definition~1 of \cite{joshi2025learning}]
     Let $\cH$ be a multi-label hypothesis class. We say that $\cH$ is \emph{realizable  multi-label learnable} if there exist a function
     $m^{\unknown}: (0,1) \times (0,1) \to \bN$, and a learner $\cA: \left(\cX \times \cY \right)^* \rightarrow \Pi(\cY)^\cX$ such that for any $h^* \in \cH$, and distribution $\cD$ over $\cX \times \cY$ realizable by $h^*$, any $(\varepsilon, \delta) \in (0,1) \times (0,1)$, if $m \geq m^{\unknown}(\varepsilon, \delta)$, with probability at least $1 - \delta$ over i.i.d.~generated sample $U = \left((x_1, y_1), \ldots, (x_m, y_m)\right)$ from  $\cD$ we have 
     \[
    \probs{x\sim \cD_{\cX}}{\mathcal{A}(U)(x) \notin h^*(x)}< \varepsilon,
    \]
    \azedit{where $\cD_{\cX}$ is the marginal distribution of $\cD$.}

\end{definition}

The following corollary exhibits that the sample complexity of PAC learning with multiple labels in the realizable setting is controlled by $\LDU$. The proof simply follows by combining Lemma~\ref{lem:freedman} and the mistake upper bound of $\LDU(\cH)$ in Theorem~\ref{thm:real-char}. Noting that for any finite class $\cH$, we have $\LDU(\cH)\leq \log_2(|\cH|)$ (see Corollary~\ref{corollary:LD_finite}), the following generalizes Theorem~5 of \citet{joshi2025learning} to infinite hypothesis classes, resolving one of their open questions. 
\begin{corollary}\label{cor:realizable_multi_label}

Every multi-label hypothesis class with bounded $\LDU$ is realizable multi-label learnable with sample complexity
\[
m^{\unknown}(\varepsilon, \delta)=O\left(\varepsilon^{-1}\left(\LDU(\cH)+\log (1 / \varepsilon \delta)\right)\right)
\]
\end{corollary}

\fmreplace{Next, we consider the more relaxed setting of set-valued PAC learnability. As it is expected, we show that the sample complexity of realizable learnability in this setting is controlled by $\LDS$. More notably, we show that our results in Section~\ref{thm:set-valued-const}, imply agnostic set-valued PAC learnability for some hypothesis classes. \azedit{We will use the following variant of Lemma~\ref{lemma:A_o2b} for this result.} \fmedit{The proof is identical to the proof of Lemma~\ref{lemma:A_o2b}.}}{Next, we focus on set-valued PAC learnability.
Here, we need a set-valued variant of Lemma~\ref{lemma:A_o2b}, which we state below.
Since the proof is identical to that of Lemma~\ref{lemma:A_o2b}, it is not repeated here.}

\begin{lemma}\label{lemma:A_o2b_set}
    Let  $\cA: \left(\cX \times 2^{\cY}\right)^* \rightarrow \Pi(\cY)^\cX$ and $\cD$ be a distribution over $\cX \times 2^{\cY}$. 
    For any $U := \left((x_1, S_1), \ldots, (x_m, S_m)\right) \in (\cX \times 2^{\cY})^*$ define  
    $\cA_{\mathrm{o2b}}: \left(\cX \times 2^{\cY}\right)^* \rightarrow \Pi(\cY)^\cX$ as
    \[
     \forall x \in \cX: \cA_{o2b}(U)(x) \sim \Unif \big(f_1(x),\ldots,f_{m}(x) \big), \; \text{where } \; \forall t \in [m]: f_t(x) := \cA(U_{<t})(x).
    \]
    Then for any $\varepsilon, \delta \in (0,1] \times (0, 1]$,  with probability at least $1 - \delta$ over an i.i.d~generated sample $U := \left((x_1, S_1), \ldots, (x_m, S_m)\right)$ from $\cD$ we have
    \[
     \probs{(x,S)\sim \cD}{\cA_{\mathrm{o2b}}(U) \notin S} \leq \frac{1+ \sum_{t = 1}^{m} \ell_t+12 \log \left(\frac{2 \log m}{\delta}\right)}{m},
    \]
    where $\ell_t := \id \left[f_t\left(x_t\right) \notin S_t\right]$.
\end{lemma}

\fmedit{We first give the definition set-valued learning in the realizable setup. Then similar to realizable multi-label learning, in Corollary~\ref{cor:realizable_set}, we show that the sample complexity of realizable set-Valued learning is controlled by $\LDS$. Corollary~\ref{cor:realizable_set} immediately follows from Lemma~\ref{lemma:A_o2b_set} and Theorem~\ref{thm:real-char}.}

\begin{definition}[Realizable Set-Valued Learning]
     Let $\cH$ be a multi-label hypothesis class. We say that $\cH$ is \emph{realizable set-valued learnable} if there exist function
     $m^{\set}: (0,1) \times (0,1) \to \bN$, and a learner $\cA: \left(\cX \times 2^{\cY} \right)^* \rightarrow \Pi(\cY)^\cX$ such that for any $h^* \in \cH$, any distribution $\cD_{\cX}$ over $\cX$, any $(\varepsilon, \delta) \in (0,1) \times (0,1)$, if $m \geq m^{\set}(\varepsilon, \delta)$, with probability at least $1 - \delta$ over an i.i.d.~generated sample $U = \left((x_1, h^*(x_1)), \ldots, (x_m, h^*(x_m))\right)$ from $\cD$ we have
    \[
    \probs{x\sim \cD_{\cX}}{\mathcal{A}(U)(x) \notin h^*(x)}< \varepsilon.%
    \] 
\end{definition}

\begin{corollary}\label{cor:realizable_set}

Every multi-label hypothesis class with bounded $\LDS$ is realizable set-valued learnable with sample complexity
\[
m^{\set}(\varepsilon, \delta)=O\left(\varepsilon^{-1}\left(\LDS(\cH)+\log (1 / \varepsilon \delta) \right) \right)
\]
\end{corollary}

\fmreplace{We now define the agnostic setting of set-valued learning}{Finally, we define set-valued learning in the agnostic setup} and \azreplace{then show a  positive result for finite hypothesis classes that satisfy the conditions in Theorem~\ref{thm:set-valued-const}.}{provide sample complexity upper bound for this problem in Theorem~\ref{thm:set-batch-agnostic}.} 
\fmreplace{The proof of this theorem is immediate from Lemma~\ref{lemma:A_o2b_set} and the regret bound \azreplace{of $\log_2(|\cH|)$ exhibited in Theorem~\ref{thm:set-valued-const}}{in Theorem~\ref{thm:agnostic_set_valued}}}{The proof of Theorem~\ref{thm:set-batch-agnostic} combines Lemma~\ref{lemma:A_o2b_set}, Theorem~\ref{thm:set-valued-const}, and Theorem~\ref{thm:agnostic_set_valued}. The formal proof can be found in Appendix~\ref{pf_online_to_batch}.}
\begin{definition}[Agnostic Set-valued Learning]
     Let $\cH$ be a multi-label hypothesis class. We say that $\cH$ is \emph{agnostic set-valued  learnable} if there exist function
     $m^{\set}: (0,1) \times (0,1) \to \bN$, and a learner $\cA: \left(\cX \times 2^{\cY} \right)^* \rightarrow \Pi(\cY)^\cX$ such that for any distribution $\cD$ over $\cX \times 2^{\cY}$, any $(\varepsilon, \delta) \in (0,1) \times (0,1)$, if $m \geq m^{\set}(\varepsilon, \delta)$, with probability at least $1 - \delta$ over an i.i.d.~generated sample $U = \left((x_1, S_1), \ldots, (x_m, S_m)\right)$ from $\cD$ we have
    \[
    \probs{(x,S)\sim \cD}{\cA(U)(x)\notin S} < \inf_{h \in \cH} \probs{(x, S) \sim \cD} { h(x) \neq S} + \varepsilon.
    \]

\end{definition}

\azdelete{\begin{corollary}\label{cor:agnostic_set}
Any finite-sized multi-label hypothesis class $\cH$ such that
\[
\forall h,h' \in \cH \text{ and } \forall x \in \cX:\quad
2|h'(x)\setminus h(x)| \leq |h(x)|,
\] 
is agnostic set-valued learnable with sample complexity
\[
\azreplace{m^{\set}(\varepsilon, \delta)=O\left(\varepsilon^{-1}\left(\log |\cH| +\log (1 / \varepsilon \delta) \right)\right)}{O\left(\varepsilon^{-1}\left(\log |\cH| \right) + \varepsilon^{-2} \log(1 / \delta) \right)}.
\]
\end{corollary}}
\azdelete{We now give a theorem for agnostic learning of any hypothesis class with finite $\LDS$ in the set-valued setup. The proof of this theorem appears in Appendix~\ref{pf_online_to_batch}.}
\begin{theorem}\label{thm:set-batch-agnostic}
    Every multi-label hypothesis class with bounded $\LDS$ is realizable set-valued learnable with sample complexity
\[
O\left(\frac{\LDS(\cH)}{\varepsilon^2}\log\left(\frac{\LDS(\cH)}{\varepsilon^2}\right) + \frac{\log(1/\delta)}{\varepsilon^2}\right).
\]

Moreover, for any finite-sized multi-label hypothesis class $\cH$ such that
\[
\forall h,h' \in \cH \text{ and } \forall x \in \cX:\quad
2|h'(x)\setminus h(x)| \leq |h(x)|,
\] 
is agnostic set-valued learnable with sample complexity
\[
\azreplace{m^{\set}(\varepsilon, \delta)=O\left(\varepsilon^{-1}\left(\log |\cH| +\log (1 / \varepsilon \delta) \right)\right)}{O\left(\varepsilon^{-1}\left(\log |\cH| \right) + \varepsilon^{-2} \log(1 / \delta) \right)}.
\]
\end{theorem}
}{}

\section{Conclusion}
We studied learning from multiple correct answers. We formalized this problem in the online setting and under three models of feedback: mistake-unknown, mistake-known, and set-valued. For each model we characterized the mistake bound in the realizable setup, and we derived bounds on the regret in the agnostic setup. Our results demonstrate fundamental differences between this problem and the single-label setting.
We further show how our results yield novel sample complexity bounds for the corresponding problems in the batch setup.

Our goal in this work is to model recent emerging uses of automated prediction. The formal framework(s) and the results stated in this paper provides insight to theoretically understanding such tasks. Many relevant aspects, however, are left unaddressed. Examples include non-binary evaluation of outputs, implicitly used prior knowledge beyond the framework of hypothesis classes, the computational complexity of learning and more.

\section*{Acknowledgements}
Shai Ben-David was supported by NSERC through a Discovery grant and by the Province of Ontario, the Government of
Canada through CIFAR, and companies sponsoring the Vector Institute. Alireza F. Pour and Farnam Mansouri were supported by Vector Institute Research Grant and Cheriton Graduate Scholarship.

{
  \printbibliography
}

\appendix
\section{Missing Proofs from Section~\ref{sec:LDs}}
\subsection{Proof of Theorem~\ref{thm:LD-0-1}: Characterization of Classes with Dimension Zero} \label{pf_LD_zero}

We first establish a simple structural lemma in the mistake-unknown setup, which we use to prove the theorem.

\begin{lemma} \label{lem:three-last-lay}
    If $\LDU(\cH) \geq d$, then there exists a multi-label tree $\cT$ that is $d$-shattered by $\cH$, such that for every path $\sigma = (\hat y_1, ..., \hat y_T)$ from root to leaf there exists a $h_\sigma$ that in addition to satisfying properties (i) and (ii) also satisfies $\hat y_T \notin h_\sigma(\nu_{x}(\sigma_{<T }))$.
\end{lemma}
\begin{proof}
     Consider any tree $\cT$ that is $d$-shattered by $\cH$. Suppose there exists a root to leaf path $\sigma = (\hat y_1, ..., \hat y_T)$ with  \azreplace{$\hat y_T \notin h_\sigma(\nu_{x}(\sigma_{< T}))$}{$\hat y_T \in h_\sigma(\nu_{x}(\sigma_{< T}))$}. Note that the function $h_{\sigma}$ still satisfies properties (i) and (ii) for the path $\sigma' = (\hat y_1, ..., \hat y_{T-1})$. This implies that if we remove all the children of the node $\nu(\sigma_{ \leq T - 1})$, the reduced tree will still be $d$-shattered by $\cH$. 
\end{proof}

\LDzero*
\begin{proof}
    Since $\LDS(\cH) \leq \LDK(\cH) \leq \LDU(\cH)$  it suffice to show that $\forall x\in \cX: \bigcap_{h \in \cH} h(x) \neq \emptyset$ implies $\LDU(\cH) = 0$ and $\LDS(\cH) = 0$ implies $\forall x\in \cX: \bigcap_{h \in \cH} h(x) \neq \emptyset$.
    
    \emph{$\forall x\in \cX: \bigcap_{h \in \cH} h(x) \neq \emptyset \Longrightarrow \LDU(\cH) = 0$}: for the sake of contradiction suppose that $\LDU(\cH) \geq 1$. Let $\cT$ be a tree promised in Lemma~\ref{lem:three-last-lay} such that $\cT$ is $\LDU(\cH)$-shattered by $\cH$. Consider a node of tree such that all its children are a leaf. Denote $x$ as the node's $\cX-$value. Due to Lemma~\ref{lem:three-last-lay}, for every $\hat r \in \cY$, there exists a $h_{\hat r} \in \cH$ such that $\hat r \notin h_{\hat r}(x)$. Thus $\bigcap_{h \in \cH} h(x) = \emptyset$, which is a contradiction.

    \emph{$\LDS(\cH) = 0 \Longrightarrow \forall x\in \cX: \bigcap_{h \in \cH} h(x) \neq \emptyset$}: for the sake of contradiction suppose that there exists $x \in \cX$ such that for every $\hat r \in \cY$ there exists a $h_{\hat r} \in \cH$ such that $\hat r\notin h_{\hat r}(x)$. 
    Let $\cT$ be a set-valued tree of depth 1 with root $x$ and $\forall \hat r \in \cY: \nu_S(\hat r) =  h_{\hat r}(x)$. Then $\cT$ is shattered by $\cH$, which implies $\LDS(\cH) \geq 1$ which is a contradiction.
\end{proof}
\ifthenelse{\boolean{arxivsubmission}}{}{\subsection{Proof of Theorem~\ref{thm:real-char}: Realizable Characterizations
}\label{pf_realizable_online}

\begin{algorithm}[tb!]
\caption{Mistake Unknown Multi-label Standard Optimal Algorithm (SOA)}
\begin{algorithmic}[1] 
\Require Multi-label hypothesis class $\cH$ and $T \in \bN$.
\State Initialize $\azreplace{V_0}{V_1} = \cH$, $\azreplace{\mof_0}{\mof_1} \equiv 0$.
\For{$t = 1$ \textbf{to} $T$}
    \State Receive $x_t$ 
    \State For $r \in \cY$, define $V^{(r)}_t := \{h \in V_t: r \in h(x_t)\}$
    \State For every $\hat r \in \cY$ and $h \in \cH$ define $\mof^{(\hat r)}_t(h) := \mof_t(h) + \id [\hat r \notin h(x_t)]$

    \State Predict $\hat y_t := \argmin{\hat{r} \in \cY} \max_{r\in \cY} \LDU(V^{(r)}_t, \mof^{(\hat{r})}_t)$   
    \State Receive label $y_t$
    \State Set $V_{t + 1} = V^{(y_t)}_t$, and $\mof_{t + 1} =\mof^{(\hat y_t)}_t$.
\EndFor
\end{algorithmic}
\label{alg:mistake-unknown-SOA}
\end{algorithm}

\begin{algorithm} [tb!]
\caption{Mistake Known Multi-label Standard Optimal Algorithm (SOA)}
\begin{algorithmic}[1] 
\Require Multi-label hypothesis class $\cH$ and $T \in \bN$.
\State Initialize $\azreplace{V_0}{V_1} = \cH$.
\For{$t = 1$ \textbf{to} $T$}
    \State Receive $x_t$
    \State For $r, \hat r \in \cY$, define $V^{(r, \hat r, 1)}_t := \{h \in V_t: r \in h(x_t), \hat r \notin h(x_t) \}$ and $V^{(r, \hat r, 0)}_t := \{h \in V_t: r \in h(x_t), \hat r \in h(x_t) \}$
    \If{$\LDK(V_t) > 0$}
        \State Predict $\hat y_t := \argmin{\hat{r} \in \cY} \max_{r\in \cY} \LDK(V^{(r, \hat r, 1)}_t)$
        
    \Else
        \State Predict any $\hat y_t \in \bigcap_{h \in V_t} h(x_t)$ (exists due to Theorem~\ref{thm:LD-0-1})
    \EndIf
     \State Receive label $y_t$ and $b_t$ indicating whether $\hat y_t$ was mistake
    \State Set $V_{t + 1} = V^{(y_t, \hat y_t, b_t)}_t$
\EndFor
\end{algorithmic}
\label{alg:mistake-known-SOA}
\end{algorithm}

\begin{algorithm} 
\caption{Set-valued Standard Optimal Algorithm (SOA)}
\begin{algorithmic}[1] 
\Require Multi-label hypothesis class $\cH$ and $T \in \bN$.
\State Initialize $\azreplace{V_0}{V_1} = \cH$.
\For{$t = 1$ \textbf{to} $T$}
    \State Receive $x_t$ 
    \State For $S \subseteq \cY$ and $\hat r \in \cY$ define $V^{(S, \hat r)}_t := \{h \in V_t:  h(x_t) = S, \hat r \notin h(x_t) \}$
    \If{$\LDS(V_t) > 0$}
        \State Predict $\hat y_t := \argmin{\hat{r} \in \cY} \max_{S \subseteq \cY} \LDS(V^{(S, \hat r)}_t)$
        
    \Else
        \State Predict any $\hat y_t \in \bigcap_{h \in V_t} h(x_t)$ (exists due to Theorem~\ref{thm:LD-0-1})
    \EndIf
    \State Receive set feedback $S_t$ indicating whether $\hat y_t$ was mistake
    \State Set $V_{t + 1} = V^{(S_t, \hat y_t)}_t$
\EndFor
\end{algorithmic}
\label{alg:set-SOA}
\end{algorithm}

\RealChar*
\begin{proof}
We prove the characterization for each feedback model in the following.

    \textbf{Mistake-unknown characterization:} We first prove that for any $T \geq \azreplace{\LDK(\cH)}{\LDU(\cH)}$ and any algorithm $\cA$ we have $\rM^{\unknown}_\cA(T, \cH) \geq \LDU(\cH)$. Let $\cT$ be a multi-label tree that is $\LDU(\cH)$-shattered by $\cH$ and let $x_1$ denote its root. For every $t \in  [\LDU(\cH)]$, denote $\hat y_{t} := p_{\cA}(t)$, $y_{t} = \nu_y (\hat y_1, \ldots, \hat y_{t})$, and $x_{t+1} = \nu_x (\hat y_1, \ldots, \hat y_{t})$. Let $h^{\star} = h_{\sigma^{\star}}$ for $\sigma^{\star} = (\hat y_1, \ldots, \hat y_{\LDK(\cH)})$. From property (i) of Definition~\ref{def:unknown-multi-littlestone} we have $y_{t} \in h^{\star}(x_t)$. Moreover, due to property (ii) $\sum_{t = 1}^{\LDU(\cH)} \id [ \hat y_{t} \notin h^{\star}(x_t) ] \geq \LDU(\cH)$. This concludes that $\cA$ makes at least $\LDU(\cH)$ mistakes and the claim follows.

Next we show that for the learner $\cA$ defined in Algorithm~\ref{alg:mistake-unknown-SOA} and any $T \in \bN$ we have $\rM^{\unknown}_\cA(T, \cH) \leq \LDU(\cH)$. We prove this by showing that $\LDU(V_t, \mof_t)$ is non-increasing in $t$. Note that this finishes the proof since for $h^\star \in \cH$,
\[
\begin{aligned}
    \LDU(\cH) &= \LDU(\azreplace{V_0}{V_1}, \azreplace{\mof_0}{\mof_1}) \\
    &\geq \LDU(\azreplace{V_T}{V_{T+1}}, \azreplace{\mof_T}{\mof_{T+1}})  && (\text{since $\LDU(V_t, \mof_t)$ is non-increasing})\\
    & \geq \sup_{h \in \azreplace{V_T}{V_{T+1}}} \azreplace{\mof_T}{\mof_{T+1}}(h) && (\text{due to Claim~\ref{claim:sup_offset}})\\
    & \geq \azreplace{\mof_T}{\mof_{T+1}}(h^\star) \\
    & = \rM^{\unknown}_\cA(T, \cH),
\end{aligned}
\]
where the last equality follows from the fact that we can inductively prove 
\[
\azreplace{\mof_T}{\mof_{T+1}}(h^{\star}) = \azreplace{\mof_{T-1}}{\mof_T}(h^{\star}) +\id[\hat{y}_\azreplace{{T-1}}{T} \notin h^{\star}(x_\azreplace{{T-1}}{T})] = \sum_{t=1}^T \id[\hat{y}_{t} \notin h^{\star}(x_t)] = \sum_{t=1}^T \id[p_{\cA}(t) \notin h^{\star}(x_t)].
\]

We now prove that $\LDU(V_t, \mof_t)$ is non-increasing in $t$. For the sake of contradiction suppose $\LDU(V_{t + 1}, \mof_{t+1}) > \LDU(V_{t}, \mof_{t})$ for some $t \in [T]$. For every $\hat r \in \cY$ denote 
\[
y_{\mathsf{max}}(\hat r) := \argmax{r} \; \LDU(V^{(r)}_t, \mof^{(\hat{r})}_t).
\] 
Note that from the description of Algorithm~\ref{alg:mistake-unknown-SOA}, we have \azreplace{$\hat{y}_t = \arg\min_{\hat{r} \in \cY} y_{\mathsf{max}}(\hat r)$}{\[\hat{y}_t = \arg\min_{\hat{r} \in \cY} \LDU(V^{(y_{\mathsf{max}}(\hat r))}_t, \mof^{(\hat r)}_t).\]} Therefore, for all $\hat r \in \cY$, we have,
\[
\begin{aligned}
\LDU(V^{(y_{\mathsf{max}}(\hat r))}_t, \mof^{(\hat{r})}_t) & \geq \LDU(V^{(y_{\mathsf{max}}(\hat y_t))}_t, \mof^{(\hat{y}_t)}_t) \\
& \geq  \LDU(V^{(y_t)}_t, \mof^{(\hat{y}_t)}_t) \\
& = \LDU(V_{t + 1}, \mof_{t+1}). \\
\end{aligned}
\]
\azedit{For any $\hat{r} \in \cY$,} we let $\cT_{\hat r}$ be a multi-label tree with $\LDU(V^{(y_{\mathsf{max}}(\hat r))}_t, \mof^{(\hat{r})}_t) \geq \LDU(V_{t + 1}, \mof_{t+1}) $ that is shattered by $\left(V^{(y_{\mathsf{max}}(\hat r))}_t, \mof^{(\hat r)}_t\right)$, where the label on the root is $y_{\mathsf{max}}(\hat r)$. 
Construct the multi-label tree $\overline{\cT}$ such that root of $\overline{\cT}$ is $x_t$, and each edge $\hat r \in \cY$ of the root is connected to $\cT_{\hat r}$. Then $\overline{\cT}$ is shattered by $(V_t, \mof_t)$. This indicates that $\LDU(V_t, \mof_t) \geq \LDU(V_{t + 1}, \mof_{t+1})$ which is a contradiction. 

    \textbf{Mistake-known characterization:} We first show that for any $T \geq \LDK(\cH)$ and any algorithm $\cA$ we have $\rM^{\known}_\cA(T, \cH) \geq \LDK(\cH)$. Find a complete multi-label tree $\cT$ of depth $\LDK(\cH)$ that is shattered by $\cH$. Let $x_1$ be the root of $\cT$. For every $t \in  [\LDK(\cH)]$, define $\hat y_t := p_\cA(t)$, $y_t = \nu_y (\hat y_1, ..., \hat y_t)$ and $x_{t+1} = \nu_x (\hat y_1, ..., \hat y_t)$. Let $h^\star = h_{\sigma^\star}$ for $\sigma^\star = (\hat y_1, \ldots, \hat y_{\LDK(\cH)})$. First note that due to property (i) of Definition~\ref{def:known-multi-littlestone} we have $y_t \in h^\star(x_t)$. Moreover due to property (ii), $\hat y_t \notin h^\star(x_t)$. This concludes  $\rM^{\known}_\cA(T, \cH) \geq \LDK(\cH)$. 

    Next we prove that for the learner $\cA$ defined in Algorithm~\ref{alg:mistake-known-SOA} and any $T \in \bN$, we have $\rM^{\known}_\cA(T, \cH) \leq \LDK(\cH)$. We prove this by showing for each $t$, if $\LDK(V_{t}) > 0$ and $\cA$ makes a mistake, then $\LDK(V_{t+1}) \leq \LDK(V_{t}) - 1$. Note that this concludes the claim since Theorem~\ref{thm:LD-0-1} ensures that whenever $\LDK(V_t) = 0$, $\bigcap_{h \in V_t} h(x_t) \neq \emptyset$ and $\cA$ makes no more mistakes. 

    For the sake of contradiction suppose $\cA$ makes a mistake at $t$ and $\LDK(V_{t + 1}) = \LDK(V_{t})$. For every $\hat r \in \cY$ denote $y_{\mathsf{max}}(\hat r) := \argmax{r} \; \LDK(V^{(r, \hat r, 1)}_t)$. Then, noting that $m_t=1$, we have
    \[
    \begin{aligned}
        \LDK(V^{(y_{\mathsf{max}}(\hat r), \hat r, 1)}_t) &\geq \LDK(V^{(y_{\mathsf{max}}(\hat{y}_t), \hat{y}_t, 1)}_t) \geq \LDK(V^{(y_t, \hat{y}_t, 1)}_t) \\
        & = \LDK(V_{t + 1}) = \LDK(V_{t}).
    \end{aligned}
    \] 
    \azedit{For any $\hat{r}\in\cY$}, let $\cT_{\hat r}$ be a complete multi-label tree of depth $\LDK(V_t)$ that is shattered by $V^{(y_{\mathsf{max}}(\hat r), \hat r, 1)}_t$, where the label on the root is $y_{\mathsf{max}}(\hat r)$. Construct the complete multi-label tree $\overline{\cT}$ such that root of $\overline{\cT}$ is $x_t$, and each edge $\hat r \in \cY$ of the root is connected to $\cT_{\hat r}$. Clearly $\overline{\cT}$ is shattered by $\bigcup_{\hat r} V^{(y_{\mathsf{max}}(\hat r), \hat r, 1)}_t $, which implies that it is also shattered by $V_t 
    \supseteq \bigcup_{\hat r} V^{(y_{\mathsf{max}}(\hat r), \hat r, 1)}_t $. Hence, $\LDK(V_t) \geq \LDK(V_{t}) + 1$ which is a contradiction.
    
    \textbf{Set-valued characterization:} We first show that for any $T \geq \LDS(\cH)$ and any algorithm $\cA$ we have $\rM^{\set}_\cA(T, \cH) \geq \LDS(\cH)$. Find a set-valued tree $\cT$ of depth $\LDS(\cH)$ that is shattered by $\cH$. Let $x_1$ be the root of $\cT$. For every $t \in  [\LDS(\cH)]$, define $\hat y_t := p_\cA(t)$, $S_t = \nu_S (\hat y_1, ..., \hat y_t)$ and $x_{t+1} = \nu_x (\hat y_1, ..., \hat y_t)$. Let $h^\star = h_{\sigma^\star}$ for $\sigma^\star = (\hat y_1, \ldots, \hat y_{\LDS(\cH)})$. First note that due to property (i) of Definition~\ref{def:set-multi-littlestone} we have $S_t = h^\star(x_t)$. Moreover due to property (ii), $\hat y_t \notin h^\star(x_t)$. This concludes  $\rM^{\set}_\cA(T, \cH) \geq \LDS(\cH)$. 

    Next we prove that for the learner $\cA$ defined in Algorithm~\ref{alg:set-SOA} and any $T \in \bN$, we have $\rM^{\set}_\cA(T, \cH) \leq \LDS(\cH)$. We prove this by showing at each $t$, if $\LDS(V_{t}) > 0$ and $\cA$ makes a mistake, then $\LDS(V_{t+1}) \leq \LDS(V_{t}) - 1$. Note that this concludes the claim since Theorem~\ref{thm:LD-0-1} ensures that whenever $\LDS(V_t) = 0$, we have $\bigcap_{h \in V_t} h(x_t) \neq \emptyset$ and $\cA$ makes no more mistakes. 

    For the sake of contradiction suppose $\cA$ makes a mistake at $t$ and $\LDS(V_{t + 1}) = \LDS(V_{t})$. For every $\hat r \in \cY$ denote $S_{\mathsf{max}}(\hat r) := \argmax{S \subseteq \cY} \; \LDS(V^{(S, \hat r)}_t)$
    
    \[\LDS(V^{(S_{\mathsf{max}}(\hat r), \hat r)}_t) \geq \LDS(V^{(S_{\mathsf{max}}(\hat{y}_t), \hat{y}_t)}_t) \geq \LDS(V^{(S_t, \hat{y}_t)}_t)  = \LDS(V_{t + 1}) = \LDS(V_{t}).
    \] Let $\cT_{\hat r}$ be a complete multi-label tree of depth $\LDS(V_t)$ that is shattered by $V^{(S_{\mathsf{max}}(\hat r), \hat r)}_t$, where the set of labels on the root is $S_{\mathsf{max}}(\hat r)$. Construct the complete multi-label tree $\overline{\cT}$ such that root of $\overline{\cT}$ is $x_t$, and each edge $\hat r \in \cY$ of the root is connected to $\cT_{\hat r}$. Clearly $\overline{\cT}$ is shattered by $\bigcup_{\hat r} V^{(S_{\mathsf{max}}, \hat r)}_t $, which implies that it is also shattered by $V_t 
    \supseteq \bigcup_{\hat r} V^{(S_{\mathsf{max}}(\hat r), \hat r)}_t $. Hence, $\LDS(V_t) \geq \LDS(V_{t}) + 1$ which is a contradiction.     
\end{proof}}
\section{Missing Proofs from Section~\ref{sec:agnostic}}

\subsection{Proof of Theorem~\ref{thm:H3}: Regret for Example Hypothesis Class} \label{pf_H3}
\begin{algorithm}[tb!]
\caption{Set-Valued Weighted Majority Algorithm  }
\begin{algorithmic}[1] 
\Require Multi-label hypothesis class $\cH$ and $T \in \bN$.

\State Initialize $w^{(1)}(h)=1$ for all $h \in \cH$.
\For{$t = 1$ \textbf{to} $T$}
\State Receive $x_t$.
\State For every $y \in \cY$ define $v^{(t)}(y) := \sum_{h \in \cH \text { s.t. } y \in h \left(x_t\right)} w^{(t)}(h)$.
\State Output $\hat{y}_t = y$ with probability $\frac{v^{(t)}(y)}{\sum_{y \in \cY} v^{(t)}(y)}$.
\State Receive $S_t$.
\If{$\hat y_t \notin S_t$}
\State For each $h \in \cH$, update
\[
w^{(t+1)}(h) \leftarrow \begin{cases} 2w^{(t)}(h) & \text { if } h(x_t) = S_t, \\ w^{(t)}(h) & \text { otherwise. }
\end{cases}
\]
\Else
\State For each $h \in \cH$, update
\[
w^{(t+1)}(h) \leftarrow \begin{cases} w^{(t)}(h) & \text { if } h(x_t) = S_t, \\ w^{(t)}(h) / 2 & \text { otherwise.}
\end{cases}
\]
\EndIf
 
\EndFor
\end{algorithmic}
\label{alg:set-goodH}
\end{algorithm}
\HThree*
\begin{proof}
Since there is only a single input \(x\), for the sake of brevity we write \(h\) instead of \(h(x)\) for every \(h \in \mathcal{H}_3\).

\begin{enumerate}[label=(\roman*)]

\item
Fix a learner $\cA$. For each \(t \in [T]\), let \(y_t\) be sampled uniformly from \(\cY\). For each \(y \in \cY\), define
\(
L_y := \sum_{t=1}^T \Pr\!\left[p_\cA(t) = y \right].
\)
Furthermore, let
\(
\hat{y} := \argmax{y \in \cY} \; L_y .
\)
Since \(\sum_{y \in \cY} L_y = T\), it follows that \(L_{\hat y} \geq \tfrac{T}{3}\). Now, for each \(t \in [T]\), define
\[
S_t :=
\begin{cases}
\{\hat y\} & \text{if } y_t = \hat y, \\
h_{\hat y} & \text{o.w.}
\end{cases}
\]

Then we have
\begin{equation} \label{eq:H3-1}
\begin{aligned}
\sum_{t = 1}^{T} \Pr \!\left[p_\cA(t) \notin S_t \right]
&= \sum_{t = 1}^{T}
\Bigl(
\Pr \!\left[p_\cA(t) = \hat y\right] \Pr[y_t \neq \hat y]
+ \Pr \!\left[p_\cA(t) \neq \hat y\right] \Pr[y_t = \hat y]
\Bigr) \\
&= \sum_{t = 1}^{T}
\left(
\frac{2}{3} \Pr \!\left[p_\cA(t) = \hat y\right]
+ \frac{1}{3} \Pr \!\left[p_\cA(t)\neq \hat y\right]
\right) \\
&= \frac{T}{3}
+ \frac{1}{3} \sum_{t = 1}^{T} \Pr \!\left[p_\cA(t)= \hat y\right] \\
&= \frac{T}{3} + \frac{L_{\hat y}}{3}
\;\geq\; \frac{4T}{9}.
\end{aligned}
\end{equation}

On the other hand,
\[
\begin{aligned}
\expect{\inf_{h \in \cH_3} \sum_{t=1}^T 1 \left\{ h \neq  S_t \right\}}
&\leq
\expect{\sum_{t=1}^T 1 \left\{ h_{\hat y} \neq S_t \right\}} \\
&= \sum_{t=1}^T \Pr[y_t = \hat y]
= \frac{T}{3}.
\end{aligned}
\]

Combining this with \eqref{eq:H3-1} completes the proof.

\item Let $\cA$ be the algorithm defined in Algorithm~\ref{alg:set-goodH}. For every \(t \in [T]\), denote
\[
W^{(t)} := \sum_{h \in \cH_3} w^{(t)}(h).
\]

We first show that \(\expect{ W^{(t)} }\) is always non-increasing with respect to $t$. \azedit{Notice that if $S_t \notin \cH$, either the learner makes a mistake in which case the weights remain unchanged or the learner does not make a mistake in which case all weights will be cut in half. In any case $W^{(t+1)} \leq W^{(t)}$. Otherwise,} without loss of generality, suppose \(S_t = \{1,2\}\). Notice that 
\[
\Pr[\hat y_t \notin \{1, 2\}] = \frac{w^{(t)}(h_1)+w^{(t)}(h_2)}{2W^{(t)}} , \quad \Pr[\hat y_t \in \{1, 2\}] = \frac{w^{(t)}(h_1)+w^{(t)}(h_2)+2w^{(t)}(h_3)}{2W^{(t)}}
\]
Therefore,
\[
\begin{aligned}
\expect{ W^{(t+1)} \mid w^{(t)} }
&= W^{(t)}
+ w^{(t)}(h_3)\cdot
\frac{w^{(t)}(h_1)+w^{(t)}(h_2)}{2W^{(t)}} \\
&\quad
- \frac{w^{(t)}(h_1)+w^{(t)}(h_2)}{2}
\cdot
\frac{w^{(t)}(h_1)+w^{(t)}(h_2)+2w^{(t)}(h_3)}{2W^{(t)}} \\
&= W^{(t)} -
\frac{\left(w^{(t)}(h_1)+w^{(t)}(h_2)\right)^2}{4 W^{(t)}}
< W^{(t)} .
\end{aligned}
\] 
This proves the claim. Next, for every \(h \in \cH_3\), define
\[
R(h) := \sum_{t=1}^T
\id\left[p_\cA(t) \notin S_t \right]
-
\id \left[ h \neq S_t \right].
\]

Observe that by the definition of \(w^{(t+1)}\), for every \(h \in \cH_3\),
\[
\id\left[p_\cA(t)\notin S_t \right]
-
\id \left[ h \neq S_t \right]
=
\log_2\!\left(
\frac{w^{(t+1)}(h)}{w^{(t)}(h)}
\right).
\]
Therefore,
\(
R(h) = \log_2\!\left( w^{(T+1)}(h) \right).
\)
Let
\(
h^* := \argmax{h \in \cH_3} w^{(T+1)}(h).
\)
By Jensen’s inequality, we have
\[
\begin{aligned}
2^{R^{\set}_{\cA}(T,\cH_3)}
&= 2^{\expect{R(h^*)}}
\leq \expect{ 2^{R(h^*)}} 
= \expect{w^{(T+1)}(h^*)} \leq \expect{W^{(T+1)}}
\leq W^{(1)} = |\cH_3|.
\end{aligned}
\]
This completes the proof.
\end{enumerate}   
\end{proof}

\subsection{Proof of Theorem~\ref{thm:set-valued-const}: Constant Set-Valued Regret} \label{pf_agnostic_online}

\SetValConst*
\begin{proof}
We set $\cA$ to be the Algorithm defined in Algorithm~\ref{alg:set-goodH}. For every \(t \in [T]\), denote
\[
W^{(t)} := \sum_{h \in \cH} w^{(t)}(h).
\]
It suffices to show that $\expect{W^{(t)}}$ is non-increasing in $t$; the remainder of the proof then follows identically to the proof of Theorem~\ref{thm:H3}, part~(ii). Notice that, given any fixed $S_t$ we have
\[
\Pr \left[\hat y_t \in S_t \right] = \sum_{y \in S_t} \frac{v^{(t)}(y)}{V^{(t)}}
 \text{ and } \Pr \left[\hat y_t \notin S_t \right] = \sum_{y \notin S_t} \frac{v^{(t)}(y)}{V^{(t)}},\]
\azedit{where $V^{(t)} = \sum_{y\in \cY} v^{(t)(y)}$.} Therefore,
\[
\begin{aligned}
\expect{ W^{(t+1)} \mid w^{(t)} }
& = W^{(t)}
+ \sum_{h: h(x_t) = S_t} w^{(t)}(h)
    \Pr \left[\hat y_t \notin S_t \right] - \sum_{h: h(x_t) \neq S_t} \frac{w^{(t)}(h)}{2}
   \Pr \left[\hat y_t \in S_t \right]. \\
&= W^{(t)}
+ \sum_{h: h(x_t) = S_t} w^{(t)}(h)
    \sum_{y \notin S_t} \frac{v^{(t)}(y)}{V^{(t)}} 
- \sum_{h: h(x_t) \neq S_t} \frac{w^{(t)}(h)}{2}
    \sum_{y \in S_t} \frac{v^{(t)}(y)}{V^{(t)}} .
\end{aligned}
\]

Next, note that for every $h' \in \cH$, the term $w^{(t)}(h')$ appears
$|h'(x_t)\setminus S_t|$ times in $\sum_{y \notin S_t} v^{(t)}(y)$, and
$|h'(x_t)\cap S_t|$ times in $\sum_{y \in S_t} v^{(t)}(y)$. Therefore,
\begin{equation} \label{eq:set-valued-const-1}
\begin{aligned}
\expect{ W^{(t+1)} \mid w^{(t)} }
&= W^{(t)}
+ \frac{1}{V^{(t)}} \sum_{h: h(x_t) = S_t} w^{(t)}(h)
    \sum_{h' \in \cH} |h'(x_t)\setminus S_t|\, w^{(t)}(h') \\
&\quad
- \frac{1}{V^{(t)}} \sum_{h: h(x_t) \neq S_t} \frac{w^{(t)}(h)}{2}
    \sum_{h' \in \cH} |h'(x_t)\cap S_t|\, w^{(t)}(h') \\
&= W^{(t)}
+ \frac{1}{V^{(t)}} \sum_{h: h(x_t) = S_t} w^{(t)}(h)
    \sum_{h': h'(x_t) \neq S_t} |h'(x_t)\setminus S_t|\, w^{(t)}(h') \\
&\quad
- \frac{1}{V^{(t)}} \sum_{h: h(x_t) \neq S_t} \frac{w^{(t)}(h)}{2}
    \sum_{h' \in \cH} |h'(x_t)\cap S_t|\, w^{(t)}(h') ,
\end{aligned}
\end{equation}
where the last equality uses the fact that
$|h'(x_t)\setminus S_t| = 0$ whenever $h'(x_t) = S_t$.

Now observe that
\[
\begin{aligned}
\sum_{h: h(x_t) \neq S_t} \frac{w^{(t)}(h)}{2}
\sum_{h' \in \cH} |h'(x_t)\cap S_t|\, w^{(t)}(h')
&\azreplace{\leq}{\geq}
\sum_{h: h(x_t) \neq S_t} \frac{w^{(t)}(h)}{2}
\sum_{h': h'(x_t) = S_t} |S_t|\, w^{(t)}(h') \\
&=
\sum_{h: h(x_t) = S_t} w^{(t)}(h)
\sum_{h': h'(x_t) \neq S_t} \frac{|S_t|}{2}\, w^{(t)}(h'),
\end{aligned}
\]
where in the last step we exchanged the order of summation.
Combining this bound with~\eqref{eq:set-valued-const-1} and using the assumption
of the theorem, we obtain
\[
\begin{aligned}
\expect{ W^{(t+1)} \mid w^{(t)} }
&\azreplace{\geq}{\leq} W^{(t)}
+ \frac{1}{V^{(t)}} \sum_{h: h(x_t) = S_t} w^{(t)}(h)
\sum_{\azreplace{h' \in \cH}{h': h'(x_t) \neq S_t}}
\left( |h'(x_t)\setminus S_t| - \frac{|S_t|}{2} \right) w^{(t)}(h') \\
&= W^{(t)}
+ \frac{1}{V^{(t)}} \sum_{h: h(x_t) = S_t} w^{(t)}(h)
\sum_{\azreplace{h' \in \cH}{h': h'(x_t) \neq S_t}}
\left( |h'(x_t)\setminus h(x_t)| - \frac{|h(x_t)|}{2} \right) w^{(t)}(h') \\
&\azreplace{\geq}{\leq} W^{(t)} .
\end{aligned}
\]
Thus, $\expect{W^{(t)}}$ is non-increasing in $t$, which completes the proof.
\end{proof}

\subsection{Proof of Theorem~\ref{thm:agnostic_set_valued}: Set-Valued Regret Upper Bound}\label{section:pf_agnostic_set_valued}

In contrast to the mistake-unknown setting, in the set-valued setting we have full information in the feedback. In other words, the learner observes the entire set of labels $S_t$ after prediction and can determine for every expert whether they made a mistake (as opposed to the mistake-known where the mistake bit is only known for the experts that recommended the predicted label). Hence, we can leverage the tools from the prediction with expert advice literature to get tighter bounds. 

Let $T,N\in\bN$ and $\cE = \{E_1,\ldots,E_N\}$ be a set of $N$ experts such that each expert $E_i$ predicts a single label at each time $t$, which we denote by $E_i^{(t)}$. For $t\in[T],i\in[N]$ denote by $b_{i}(t):=\id[E_i^{(t)}\notin S_t]$ the bit that indicates whether $E_i$ made a mistake at time $t$. Let MWU be the algorithm that starts by setting a weight of $w_i(1) = 1$ to each expert. At each round $t\geq 2$ the learner updates $w_{i}(t) = w_{i}(t-1)e^{-\eta b_{i-1}(t)}$ for some $\eta>0$. Let $Z_t = \sum_{i\in[N]}w_{i}(t)$. MWU then randomly draws an expert $E_i$ with probability $w_{i}(t)/Z_t$ and predicts according the the selected expert.  

\begin{lemma}[Theorem~21.11 of \citet{shalev2014understanding}]\label{lemma:wmu}
    Let $\eta = \sqrt{2\log(N)/T}$ and assume $T > 2\log(N)$. For the MWU algorithm defined above we have
    \[
 \sum_{i=1}^T \id[p_{MWU}(t) \notin S_t] - \min_{i\in [N]}\sum_{i=1}^T \id[E_I{(t)}\notin S_t] \leq \sqrt{2\log(N) T}.
    \]
\end{lemma}

\SetValAgnostic*
\begin{proof}
For any subset $I \subseteq [T]$ of size at most $\LDS(\cH)$ we will create an expert $E_I$ that predicts according to the prediction of the Set-valued SOA that has in its history only the past information revealed in rounds in $I$. Let $\cI = \{I\subseteq [T]: |I|\leq \LDS(\cH)\}$. We then define $\cE = \{E_I: I\in \cI\}$ as the set of all such experts which satisfies $|\cE|\leq(eT/\LDS(\cH))^{\LDS(\cH)}$.
  Note that similar to the arguments in the proof of Theorem~\ref{thm:agnostic_mistake_unknown}, we can prove that each expert defined above is a valid expert since the SOAs just make their prediction based on the past history of $X_{<t},Y_{<t}, S_{<t}$. Moreover, similar to proof of Thoerem~\ref{thm:agnostic_mistake_unknown}, and from the guarantee that set-valued SOA will make at most $\LDS(\cH)$ mistakes on the longest realizable subsequence of $[T]$, we can see that there exists an expert $E_{I^*}$ in the set $\cE$ defined above such that 
  \[
    \sum_{t=1}^T \id [E_{i^*}^{(t)} \notin S_t] \leq \LDS(\cH) + \min_{h\in\cH}\sum_{t=1}^T \id[h(x_t) \neq S_t]. \]

Now, for the learner $\cA$ that applies MWU on these sets of experts we have from Lemma~\ref{lemma:wmu} that
\[
\sum_{i=1}^T \id[p_{\cA}(t) \notin S_t] - \min_{E_I \in \cE}\sum_{i=1}^T \id[E_I{(t)}\notin S_t] \leq \sqrt{2\LDS(\cH)T\ln\left(\frac{eT}{\LDS(\cH)}\right)}.
\]
This implies that
  \[
    \sum_{i=1}^T \id[p_{\cA}(t) \notin S_t] - \min_{h\in\cH}\sum_{t=1}^T \id[h(x_t) \neq S_t]\leq \LDS(\cH) + \sqrt{2\LDS(\cH)T\ln\left(T\right)} \]
    and concludes the proof.
\end{proof}

\ifthenelse{\boolean{arxivsubmission}}{}{\section{Missing Proofs from Section~\ref{sec:batch}: Online-to-Batch Conversion}\label{pf_online_to_batch}
{\bf Notations.} For any $h^* \in \cH$ and $f \in \cY^{\cX}$ and distribution $\cD$ over $\cX$ define
$\err_{\cD, h^*}(f) = \Pr_{x \sim \cD_{\cX}}[f(x) \notin h^*(x)]$. Moreover, for any distribution $\cD$ over $\cX \times 2^{\cY}$, define $\err_{\cD}(f) = \Pr_{(x, S) \sim \cD}[f(x) \notin S]$. For any distribution $\cD$ over $\cX \times \cZ$, we denote by $\cD_{\cX}$ the marginal distribution of $\cD$ over $\cX$. Throughout Section~\ref{sec:batch}, we instantiate this result with $\cZ$ as $\cY$ (Corollary~\ref{cor:realizable_multi_label}) or $2^{\cY}$ (Corollary~\ref{cor:realizable_set} and Theorem~\ref{thm:set-batch-agnostic}) based on the learning model we consider.

\begin{lemma}[Freedman's inequality, Theorem 3 of \cite{li2021breaking}] \label{lem:freedman} Consider a filtration $\mathcal{F}_0 \subset \mathcal{F}_1 \subset \mathcal{F}_2 \subset \ldots$.
Let $Y_m=\sum_{i=1}^m X_i$, where $(X_i)_{i\in[m]}$  is a real-valued scalar sequence such that for some constant $R<\infty$, \azdelete{and satisfies} $\left|X_i\right| \leq R$ and $\expect{X_i \mid \mathcal{F}_{i - 1}}=0$. Define the predictable variance process
\[
W_m:=\sum_{i=1}^m \expect{X_i^2 \mid \mathcal{F}_{i - 1}},
\]
and assume deterministically that $W_m \leq \sigma^2$ for some constant $\sigma^2<\infty$. Then for any integer $n \geq 1$, with probability at least $1-\delta$,
\[
\left|Y_m\right| \leq \sqrt{8 \max \left\{W_m, \frac{\sigma^2}{2^n}\right\} \log \left(\frac{2 n}{\delta}\right)}+\frac{4}{3} R \log \left(\frac{2 n}{\delta}\right) .
\]
\end{lemma}

\subsection{Proof of Lemma~\ref{lemma:A_o2b}}
    Due to definition of $f_t$, we have
    \[
    \expect{\ell_t \mid U_{<t}}=\err_{\cD_{\cX}, h^*}\left(f_t\right) .
    \]

Hence
\[
\expect{\err_{\cD_{\cX}, h^*}\left(\cA_{\mathrm{o2b}} (U) \right)}=\expect{\frac{1}{m} \sum_{t=1}^m \err_{\cD_{\cX}, h^*}\left(f_t\right)}=\expect{\frac{1}{m} \sum_{t=1}^m \ell_t}.
\]

For the high-probability statement, define the martingale differences

\[
M_t:=\err_{\cD}\left(f_t\right)-\ell_t, \quad \text { where }\left|M_t\right| \leq 1 \text { almost surely. }
\]

Then $\expect{M_t \mid U_{<t}}=0$, and
\[
\begin{aligned}
    \expect{M_t^2 \mid U_{<t}}& =\expect{\left(\err_{\cD_{\cX}, h^*}\left(f_t\right)-\ell_t\right)^2 \mid U_{<t}}=\operatorname{Var}\left(\ell_t \mid U_{<t}\right)\\
    & =\err_{\cD_{\cX}, h^*}\left(f_t\right)\left(1-\err_{\cD_{\cX}, h^*}\left(f_t\right)\right) \leq \err_{\cD_{\cX}, h^*}\left(f_t\right) .
\end{aligned}
\]

And taking $W_m=\sum_{t=1}^m \err_{\cD_{\cX}, h^*}\left(f_t\right)$ and $\sigma^2=m$ suffices, thus, using Lemma~\ref{lem:freedman} with $n=\log m$ inequality gives us with probability $1-\delta$
\[\begin{aligned}
\sum_{t=1}^m M_t & \leq \sqrt{8\left(1+\sum_{t=1}^m \err_{\cD_{\cX}, h^*}\left(f_t\right)\right) \log \left(\frac{2 \log m}{\delta}\right)}+\frac{4}{3} \log \left(\frac{2 \log m}{\delta}\right) \\
& \leq \frac{1}{2}\left(1+\sum_{t=1}^m \err_{\cD_{\cX}, h^*}\left(f_t\right)\right)+4 \log \left(\frac{2 \log m}{\delta}\right)+\frac{4}{3} \log \left(\frac{2 \log m}{\delta}\right).
\end{aligned}
\]
Substituting $M_t$ and rearranging terms,
\[\sum_{t=1}^m \err_{\cD_{\cX}, h^*}\left(f_t\right) \leq 1+2 \sum_{t=1}^m \ell_t+\frac{32}{3} \log \left(\frac{2 \log m}{\delta}\right)
\]

Finally noting that $\err_{\cD_{\cX}, h^*}\left(\cA_{\mathrm{o2b}}(U) \right)=\frac{1}{m} \sum_{t=1}^m \err_{\cD_{\cX}, h^*}\left(f_t\right)$ completes the proof.\hfill \qed

\subsection{Proof of Theorem~\ref{thm:set-batch-agnostic}}

    We only prove the first part. The second part would be similar to the first part. Let $h^* \in \arg\min_{h \in \cH} \probs{(x, S) \sim \cD} {S \neq h(x) }$. By a Chernoff's bound we know that if $m\geq 2\log(4/\delta)/\varepsilon^2$ then with probability $1-\delta/2$ over $U\sim \cD^m$ we have 
    \begin{equation}\label{eq:chernoff_set}
    \frac{1}{m}\sum_{(x,S)\in U} \id[h^*(x) \neq S] \leq \probs{(x, S) \sim \cD} { h^*(x) \neq S} + \frac{\varepsilon}{2}.
    \end{equation}

    From Theorem~\ref{thm:agnostic_set_valued} we know that for any set $U = ((x_1,S_1)\ldots, (x_m,S_m))$ of size $m \geq \LDS(\cH)$, there exists an algorithm $\cA$ such that \[
    \frac{1}{m}\sum_{t=1}^{m} \id[\cA(U_{<t})(x) \notin S_t] - \min_{h\in\cH}\frac{1}{m}\sum_{t=1}^m \id[h(x_t) \neq S_t]\leq \frac{\LDS(\cH)}{m} + \frac{\sqrt{2\LDS(\cH)m\ln\left(m\right)}}{m}. \]
    We now invoke Lemma~\ref{lemma:A_o2b} and take a union bound with Equation~\eqref{eq:chernoff_set} to conclude that there exists a learner $\cA_{o2b}$ such that with probability at least $1-\delta$ over $U\sim \cD^m$,
    \[\begin{aligned}
     &\probs{(x,S)\sim \cD}{\cA_{\mathrm{o2b}}(U) \notin S} \leq \frac{1+  \sum_{t=1}^{m} \id[\cA(U_{<t})(x) \notin S_t]}{m}+12 \frac{\log \left(\frac{2 \log m}{\delta}\right)}{m}\\
     & \leq \frac{1}{m}+  \frac{\min_{h\in\cH}\sum_{t=1}^m \id[h(x_t) \neq S_t]+\LDS(\cH) + \sqrt{2\LDS(\cH)m\ln\left(m\right)}}{m} + 12 \frac{\log \left(\frac{2 \log m}{\delta}\right)}{m}\\
     & \leq \frac{1}{m}+  \frac{\sum_{t=1}^m \id[h^*(x_t) \neq S_t]+\LDS(\cH) + \sqrt{2\LDS(\cH)m\ln\left(m\right)}}{m} + 12 \frac{\log \left(\frac{2 \log m}{\delta}\right)}{m}\\
      & \leq \frac{1}{m}+  \probs{(x, S) \sim \cD} { h^*(x) \neq S} + \frac{\varepsilon}{2}+\frac{\LDS(\cH) + \sqrt{2\LDS(\cH)m\ln\left(m\right)}}{m} + 12 \frac{\log \left(\frac{2 \log m}{\delta}\right)}{m}.
     \end{aligned}
    \]
Letting $m \geq O\left(\frac{\LDS(\cH)}{\varepsilon^2}\log\left(\frac{\LDS(\cH)}{\varepsilon^2}\right) + \log(1/\epsilon\delta)\right)$ completes the proof.

}

\end{document}